%% file: main.tex
\documentclass{article}

\input{head}

\usepackage{microtype}
\usepackage{graphicx}
\usepackage{subfigure}
\usepackage{booktabs} 

\usepackage{hyperref}

\usepackage[accepted]{icml2024}

\usepackage{amsmath}
\usepackage{amssymb}
\usepackage{mathtools}
\usepackage{amsthm}

\usepackage[capitalize,noabbrev]{cleveref}

\theoremstyle{plain}
\newtheorem{theorem}{Theorem}[section]
\newtheorem{proposition}[theorem]{Proposition}
\newtheorem{lemma}[theorem]{Lemma}

\theoremstyle{definition}

\theoremstyle{remark}

\usepackage[textsize=tiny]{todonotes}

\icmltitlerunning{CKGConv: General Graph Convolution with Continuous Kernels}

\begin{document}

\twocolumn[
\icmltitle{CKGConv: General Graph Convolution with Continuous Kernels}




\icmlsetsymbol{equal}{*}
\icmlsetsymbol{intern}{$^\dagger$}

\begin{icmlauthorlist}
\icmlauthor{Liheng Ma}{mcgill,mila,ills,intern}
\icmlauthor{Soumyasundar Pal}{huawei}
\icmlauthor{Yitian Zhang}{mcgill,mila,ills,intern}
\icmlauthor{Jiaming Zhou}{huawei}
\icmlauthor{Yingxue Zhang}{huawei}
\icmlauthor{Mark Coates}{mcgill,mila,ills}
\end{icmlauthorlist}

\icmlaffiliation{mcgill}{Department of ECE, McGill University, Montreal, Canada}
\icmlaffiliation{mila}{Mila - Quebec AI Institute, Montreal, Canada} 
\icmlaffiliation{ills}{ILLS - International Laboratory on
Learning Systems, Montreal, Canada}
\icmlaffiliation{huawei}{Huawei Noah's Ark Lab, Montreal, Canada}

\icmlcorrespondingauthor{Liheng Ma}{liheng.ma@mail.mcgill.ca}

\icmlkeywords{Graph Convolution; Graph Neural Networks; Graph Transformer; Continuous Kernel}

\vskip 0.3in
]

\newcommand{\mlh}[1]{#1}
\newcommand{\mlhc}[1]{\color{red}[Liheng:#1] 
\color{black}}



\printAffiliationsAndNotice{}{\icmlIntern} 

\begin{abstract}
The existing definitions of graph convolution, either from spatial or spectral perspectives, 
are inflexible and not unified.
Defining a general convolution operator in the graph domain is challenging due to the lack of canonical coordinates, the presence of irregular structures, and the properties of graph symmetries.
In this work, we propose a 
\mlh{novel and general} graph convolution framework by parameterizing the kernels as continuous functions of pseudo-coordinates derived via graph positional encoding.  We name this Continuous Kernel Graph Convolution (CKGConv).
Theoretically, we demonstrate that CKGConv is flexible and expressive.
CKGConv encompasses many existing graph convolutions, and \mlh{exhibits a stronger expressiveness, as powerful as graph transformers} in terms of distinguishing non-isomorphic graphs.
Empirically, we show that CKGConv-based Networks outperform existing graph convolutional networks and perform comparably to the best graph transformers across a variety of graph datasets. 
The code and models are publicly available at \url{https://github.com/networkslab/CKGConv}.

\end{abstract}


\input{section/intro}

\input{section/related}

\input{section/method}
\input{section/theory}

\input{section/exp}
\input{section/conclusion}

\section*{Impact Statement}

This paper presents work whose goal is to advance the field of Geometric/Graph Deep Learning.
There are many potential societal consequences of our work, none which we feel must be specifically highlighted here.

\section*{Acknowledgment}
LM is supported by the Natural Sciences and Engineering Research Council of Canada (NSERC). \\
We acknowledge
the support of the Natural Sciences and Engineering Research Council of Canada (NSERC) [funding reference number 260250]
 and of the Fonds de recherche du Qu\'{e}bec.
\\
Cette recherche a \'{e}t\'{e} financ\'{e}e par le Conseil de recherches en sciences naturelles et en g\'{e}nie du Canada (CRSNG), [num\'{e}ro de r\'{e}f\'{e}rence 260250] et par les Fonds de recherche du Qu\'{e}bec.



\bibliographystyle{icml2024}
\bibliography{ref}

\newpage
\appendix
\onecolumn
\input{section/appendix}




\end{document}

%% file: head.tex
\usepackage[dvipsnames]{xcolor} 
\usepackage{microtype}
\usepackage{graphicx}
\usepackage{subfigure}
\usepackage{booktabs} 
\usepackage{makecell}

\usepackage{hyperref}
\usepackage{multirow}
\usepackage{subfloat}
\usepackage{changepage}

\usepackage{amsmath}
\usepackage{amssymb}
\usepackage{mathtools}
\usepackage{amsthm}

\newcommand{\cV}{\mathcal{V}} 
\newcommand{\cE}{\mathcal{E}} 
 
\newcommand{\bA}{\mathbf{A}} 
\newcommand{\bD}{\mathbf{D}} 
\newcommand{\bx}{\mathbf{x}} 

\newcommand{\RR}{\mathbb{R}}

\newcommand{\bW}{\mathbf{W}} 

\newcommand{\cX}{\mathcal{X}}

\newcommand{\cD}{\mathcal{D}}

\renewcommand{\P}{\mathbf{P}}

\newcommand{\A}{\mathbf{A}}
\newcommand{\D}{\mathbf{D}}
\newcommand{\V}{\mathcal{V}}

\newcommand{\func}[1]{\mathtt{#1}}

\usepackage{tikz}
\usepackage{tikz-cd}
\usepackage[normalem]{ulem}

\usepackage{pifont}
\newcommand{\cmark}{\ding{51}}%
\newcommand{\xmark}{\ding{55}}

\usepackage{wrapfig}

\usepackage[export]{adjustbox}

\usepackage{tikz}
\usetikzlibrary{automata, positioning, arrows, calc}

\tikzset{
	-,  
	>=stealth, 
	shorten >=2pt, shorten <=2pt, 
	node distance=2.5cm, 
	every state/.style={draw=blue!55,very thick,fill=blue!20}, 
	initial text=$ $, 
 }

\usepackage{soul}

\usepackage{enumitem}

\renewcommand{\cite}[1]{\citep{#1}}

\usepackage[capitalize,noabbrev]{cleveref}


%% file: section/intro.tex
\section{Introduction}
Recent advances in applying Transformer architectures in computer vision ignited a competition with the predominant Convolutional Neural Networks (ConvNets)~\cite{he2016DeepResidualLearning, tan2019EfficientNetRethinkingModel}. This rivalry started when Vision Transformers (ViTs)~\cite{dosovitskiy2021ImageWorth16x16,wang2021PyramidVisionTransformer, liu2021SwinTransformerHierarchical, liu2022SwinTransformerV2} exhibited impressive empirical gains over the \mlh{best} ConvNet architectures of the time.
However, several recent ConvNet variants~\cite{liu2022ConvNet2020s, woo2023ConvNeXtV2CoDesigning} achieve performance \mlh{comparable to that of} ViTs by incorporating innovative designs such as larger kernels and depthwise convolutions~\cite{chollet2017XceptionDeepLearning}.


\begin{figure}[t!]
    \centering
    \vskip -0.1in
\includegraphics[width=0.37\textwidth]{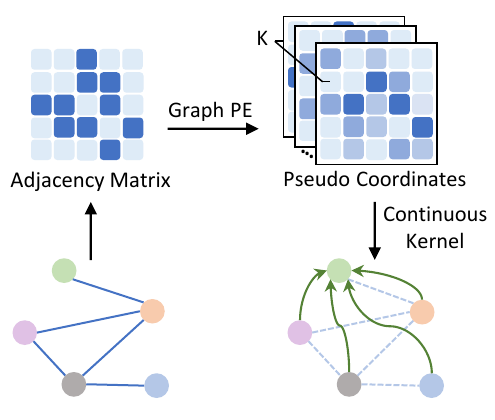}
\vspace{-0.1in}
    \caption{Continuous Kernel Graph Convolution (CKGConv)}
    \label{fig:structure}
    \vskip -0.15in
\end{figure}

In contrast, the appeal of 
Convolutional Graph Neural Networks (GNNs) 
seems to be diminishing;
Graph Transformers (GTs) demonstrate elevated efficacy on many challenging graph learning tasks~\cite{ying2021TransformersReallyPerform, rampasek2022RecipeGeneralPowerful, zhang2023RethinkingExpressivePower, ma2023GraphInductiveBiases}.
One reason might be that\mlh{, unlike convolutions in Euclidean space, }
most existing definitions of graph convolution are inflexible and not unified.
Message-passing Neural Networks (MPNNs)~\cite{gilmer2017NeuralMessagePassing, velickovic2022MessagePassingAll} are defined in the spatial domain and limited to a one-hop neighborhood;
Spectral GNNs~\cite{bruna2014SpectralNetworksLocally} are defined from a graph-frequency perspective
and require careful designs (e.g., polynomial approximation~\cite{defferrard2016ConvolutionalNeuralNetworks, wang2022HowPowerfulAre} or a \mlh{sophisticated} transformer-encoder~\cite{bo2023SpecformerSpectralGraph}) to generalize to unseen graphs.

\mlh{Unlike the general Euclidean convolution operators, there is no convolution operator for graphs that permits flexible determination of the support.}
Defining a general convolution operator in the graph domain is challenging due to the characteristics of a graph: the lack of canonical coordinates, the presence of irregular structures, and graph symmetries. These aspects are fundamentally different from Euclidean spaces.
By addressing the aforementioned challenges,
we generalize the continuous kernel convolution~\cite{romero2022CKConvContinuousKernel} to graph domain, 
and propose a general convolution framework, namely \emph{Continuous Kernel Graph Convolution} (CKGConv, as shown in Fig.~\ref{fig:structure}).
This subsumes several common convolutional GNNs, including non-dynamic\footnote{We use the term \emph{dynamic} to denote filters/kernels, generated dynamically conditioned on an input~\cite{ jia2016DynamicFilterNetworks}. \emph{Filter} and \emph{kernel} are used interchangeably in this work.}
MPNNs,
polynomial spectral GNNs, and diffusion-enhanced GNNs.
We propose three designs to address the challenges of the graph domain:
(1) pseudo-coordinates for graphs via positional encoding (addressing the lack of canonical coordinates and handling symmetries); 
(2) numerically stable scaled-convolution (compensating for irregular structure);
(3) adaptive degree-scalers (mitigating the impact of symmetries).

CKGConv is endowed with several desirable properties. It exhibits immunity against over-smoothing and over-squashing. It encompasses Equivariant Set Neural Networks~\cite{segol2020UniversalEquivariantSet} as a special case for the setting where there is no observed graph.
Furthermore, we theoretically prove that CKGConv with generalized distance (GD) can be as powerful as Generalized Distance Weisfeiler-Lehman (GD-WL)~\cite{zhang2023RethinkingExpressivePower} in graph isomorphism tests (and is thus more powerful than 1-WL).
Our experiments demonstrate the effectiveness of the proposed framework.
CKGConv ranks among the best-performing models in a variety of graph learning benchmarks and significantly outperforms existing convolutional GNNs.

Our contributions are summarized as follows:
\begin{itemize}[leftmargin=*,noitemsep,topsep=0pt]
    \item We propose a novel graph convolution framework based on \textit{pseudo-coordinates} of graphs and \textit{continuous convolution kernels}.
    
    \item We demonstrate theoretically and empirically that our proposed graph convolution is expressively powerful, \textit{outperforming existing convolutional GNNs} and achieving \textit{comparable performance to the state-of-the-art graph transformers}. 
    \item Various exploratory studies illustrate that \mlh{CKGConv displays different and potentially complementary behavior \textit{from graph transformers} that employ attention mechanisms.}
    This motivates the combination of CKGConv and attention mechanisms as a potential direction towards designing more powerful graph models. 
\end{itemize}

%% file: section/related.tex
\section{Related Work}
\textbf{Message-Passing Neural Networks} \emph{MPNNs}~\cite{gilmer2017NeuralMessagePassing, velickovic2022MessagePassingAll} are a class of GNNs widely used in graph learning tasks.
To update the representation of a node $i$, MPNNs aggregate the features from the \mlh{direct} neighbors of $i$.
In most cases, the message-passing mechanisms can be viewed as convolution with kernels locally supported on a one-hop neighborhood. The kernels may be
fixed~\cite{kipf2017SemiSupervisedClassificationGraph, xu2019HowPowerfulAre, hamilton2017InductiveRepresentationLearning}, learnable~\cite{monti2017GeometricDeepLearning},
or dynamic~\cite{velickovic2018GraphAttentionNetworks, bresson2018ResidualGatedGraph}.
Recent 
\emph{Graph Rewiring} techniques extend MPNNs beyond one-hop by introducing additional edges, guided by curvature~\cite{topping2022UnderstandingOversquashingBottlenecks}, spectral gap~\cite{arnaiz-rodriguez2022DiffWireInductiveGraph}, geodesic distance~\cite{gutteridge2023DRewDynamicallyRewired}, or positional encoding (PE)~\cite{gabrielsson2023RewiringPositionalEncodings}. Despite some similarity to our work, the usage of PE by~\citet{gabrielsson2023RewiringPositionalEncodings} is constrained by the MPNN framework and lacks flexibility.

\textbf{Spectral and Polynomial Graph Neural Networks} In contrast to MPNNs,
\emph{Spectral Graph Neural Networks} define graph filters in the spectral domain.
The pioneering approach by~\citet{bruna2014SpectralNetworksLocally} cannot generalize to an unseen graph with a different number of nodes.
Follow-up works, constituting the class of \emph{Polynomial Spectral Graph Neural Networks}, address this by 
(approximately)
parameterizing the spectral filters by a polynomial of the (normalized) graph Laplacian~\cite{defferrard2016ConvolutionalNeuralNetworks, he2021BernNetLearningArbitrary, liao2022LanczosNetMultiScaleDeep, wang2022HowPowerfulAre}.
Similarly, \emph{Diffusion Enhanced Graph Neural Networks} extend spatial filters beyond one-hop by a polynomial of diffusion operators (e.g., adjacency matrix, random walk matrix) ~\cite{gasteiger2019DiffusionImprovesGraph, gasteiger2019PredictThenPropagate,  chien2021AdaptiveUniversalGeneralized, zhao2021AdaptiveDiffusionGraph, frasca2020SIGNScalableInception, chamberlain2021GRANDGraphNeural}, exhibiting strong connections to spectral GNNs.
Notably, besides the polynomial approach, 
recent works endeavor to generalize spectral GNNs by introducing extra graph-order invariant operations on eigenfunctions of graph Laplacian~\cite{beani2021DirectionalGraphNetworks, bo2023SpecformerSpectralGraph}.

\textbf{Graph Transformers} \emph{GTs}
are equipped with the transformer architecture~\cite{vaswani2017AttentionAllYou}, consisting of self-attention mechanisms (SAs) and feed-forward networks.
Directly migrating the transformer architecture to graph learning tasks
cannot properly utilize the topological information in graphs 
and 
leads to poor performance~\cite{dwivedi2021GeneralizationTransformerNetworks}. 
Modern graph transformer architectures address this by integrating message-passing mechanisms~\cite{kreuzer2021RethinkingGraphTransformers, chen2022StructureAwareTransformerGraph, rampasek2022RecipeGeneralPowerful} or incorporating graph positional encoding (PE)~\cite{kreuzer2021RethinkingGraphTransformers, ying2021TransformersReallyPerform,  zhang2023RethinkingExpressivePower,
ma2023GraphInductiveBiases}.
Appendix~\ref{appx:graph_pe} provides more detail about graph PE.

\textbf{Expressiveness of Graph Neural Networks}
Graph isomorphism tests have been widely used to measure the theoretical expressiveness of GNNs in terms of their ability to encode topological patterns.
Without additional elements, MPNNs' expressive power is bounded by first-order  Weisfeiler-Lehman ($1$-WL) algorithms~\cite{xu2019HowPowerfulAre}. Polynomial spectral GNNs are as powerful as $1$-WL algorithms; it is not known if this is a bound~\cite{wang2022HowPowerfulAre}.
Higher-order GNNs~\cite{morris2019WeisfeilerLemanGo} \mlh{can reach the same expressive power} as $K$-WL
\mlh{with a cost of $O(N^K)$ computational complexity.}
Recently, \citet{zhang2023RethinkingExpressivePower} demonstrate that,
\mlh{under the $O(N^2)$ complexity constraint,} 
Graph Transformers with generalized distance (GD) can go beyond $1$-WL but are still bounded by $3$-WL.

\textbf{Continuous Kernel Convolution in Euclidean Spaces}
In order to handle irregularly sampled data and data at different resolutions,
\citet{romero2022CKConvContinuousKernel} and \citet{ knigge2023ModellingLongRange}  propose learning a convolution kernel as a continuous function, parameterized by a simple neural network, of the coordinates (relative positions), resulting in \emph{Continuous Kernel Convolution}.
This enables the convolution to generalize to any arbitrary support size with the same number of learnable parameters.
Driven by different motivations, 
several works have explored similar ideas for \emph{point cloud data}~\cite{hermosilla2018MonteCarloConvolution, wu2019PointConvDeepConvolutional, xu2018SpiderCNNDeepLearning,
hua2018PointwiseConvolutionalNeural}.
From a broader perspective, 
continuous kernels can be viewed as a subdomain of \emph{Implicit Neural Representation}~\cite{mildenhall2020NeRFRepresentingScenes, sitzmann2020ImplicitNeuralRepresentations, tancik2020FourierFeaturesLet}, where the representation targets are the convolution kernels.
Note that these techniques \emph{rely on canonical coordinates} in Euclidean spaces and \emph{cannot be directly applied to non-Euclidean domains} like graphs.

%% file: section/method.tex
\section{Methodology}

\subsection{Preliminary: Continuous Convolution Kernels}

Let $x: \mathbb{Z} \to \RR$ and $\psi: \mathbb{Z} \to \RR$ be two scalar-valued real sequences sampled on the set of integers $\mathbb{Z}$, where $x[k]$ and $\psi[k]$ denote the signal and filter impulse response (kernel) at time $k$, respectively. 
The discrete convolution between the signal and the kernel at time $k$ is defined as follows:
\begin{align}
(x \star \psi)[k] := \sum_{\ell \in \mathbb{Z}} x[\ell] \psi[k-\ell]\,,\label{eq:conv_iir}
\end{align}
In most cases, the kernel is of finite width $n_{\psi}$, i.e., $\psi[k]=0$ if $k<0$ or $k\geqslant n_{\psi}$. 
The convolution sum in Eq.~\eqref{eq:conv_iir} is accordingly truncated at $\ell \notin [k-n_{\psi}+1, k]$.  
However, learning such fixed support, discrete kernels cannot be generalized to arbitrary widths (i.e., different $n_{\psi}$) with the same set of parameters.

To address this shortcoming, ~\citet{romero2022CKConvContinuousKernel} propose to learn convolutional kernels by parameterizing $\psi[k]$ via a continuous function of $k$, implemented using a small neural network (e.g., multi-layer perception (MLP)). This is termed Continuous Kernel Convolution (CKConv).  
This formulation allows CKConv to model long-range dependencies and handle irregularly sampled data in Euclidean spaces.

\subsection{Graph Convolution with Continuous Kernels}

In the graph domain, convolution operators are required to handle varying sizes of supports, due to a varying number of nodes and the irregular structures.
We explore the potential of continuous kernels for graphs
and propose a general graph convolution framework with continuous kernels.

The generalization of continuous kernels to graph domain is not trivial due to the following characteristics:
(1) the \emph{lack of canonical coordinates}~\cite{bruna2014SpectralNetworksLocally} makes it difficult to define the relative positions between non-adjacent nodes;
(2) the \emph{irregular structure} requires the kernel to generalize to different support sizes while retaining numerical stability~\cite{velickovic2020NeuralExecutionGraph, corso2020PrincipalNeighbourhoodAggregation};
(3) the presence of \emph{graph symmetries} demands that the kernel can distinguish between nodes in the support without introducing permutation-sensitive operations~\cite{hamilton2017InductiveRepresentationLearning} or extra ambiguities~\cite{lim2023SignBasisInvariant}.

These challenges drive us to propose a general graph convolution framework with continuous kernels, namely CKGConv.
Our overall design consists of three innovations:
(1) we use graph positional encoding to derive the pseudo-coordinates of nodes and define relative positions;
(2) we introduce a scaled convolution to handle the irregular structure; 
(3) we incorporate an adaptive degree-scaler to improve the representation of structural graph information.

\subsubsection{Graph Positional Encoding as Pseudo-coordinates}

\newcommand{\rw}{\mathbf{M}}

In contrast to Euclidean spaces, 
the graph domain is known to lack canonical coordinates.
Consequently, it is not trivial to define the relative distance between two non-adjacent nodes.
Pioneering work~\cite{monti2017GeometricDeepLearning} attempted to define \emph{pseudo-coordinates} for graphs, however, the definition was restricted to \mlh{one-hop neighborhoods.}

In this work, we reveal that
\mlh{pseudo-coordinates can be naturally defined by graph positional encoding (PE), 
allowing us to specify relative positions for continuous kernels beyond the one-hop neighborhood constraint.}
Specifically, we use Relative Random Walk Probabilities (RRWPs)~\cite{ma2023GraphInductiveBiases}, 
which have been demonstrated to be one of the most expressive graph positional encodings~\cite{black2024ComparingGraphTransformers}.
Let $\bA \in \RR^{N \times N}$ be the adjacency matrix of a graph $\mathcal{G}=(\cV, \cE)$ with $N$ nodes, and let $\D$ be the diagonal degree matrix, $\bD_{i,i}=\sum_{j \in \cV} \bA_{i,j}$. 
The random walk matrix is $\rw := \D^{-1}\A$. Entry $\rw_{ij}$ is then the probability of a move from node $i$ to node $j$ in one step of a simple random walk.
The (top) $K$-RRWP for each pair of nodes $i, j \in \V$ consists of zeroth to $(K{-}1)$th powers of random walk matrix, defined as:
\begin{align}
    \P_{i,j} = [\mathbf{I}, \rw, \rw^2, \dots, \rw^{K-1}]_{i,j} \in \mathbb{R}^K\,,
\end{align}
where $\mathbf{I} \in \RR^{N \times N}$ denoting the identity matrix. 
We add an extra re-scaling on RRWP to remove the dependency on graph-orders (details in Appendix~\ref{appx:re-normalization}).

RRWP is not the only choice for constructing pseudo-coordinates. 
One can use other graph positional encodings such as shortest-path-distance (SPD)~\cite{ying2021TransformersReallyPerform} and resistance distance (RD)~\cite{zhang2023RethinkingExpressivePower}.

\subsubsection{Numerically Stable Graph Convolution}
\label{sec:inv_conv}


When applying a kernel to different nodes in graphs, the support size can \mlh{vary remarkably}. 
\mlh{To ensure numerical stability and the ability to generalize}, it is crucial
to avoid 
disproportionate scaling of different node representations~\cite{velickovic2020NeuralExecutionGraph}. 

\mlh{Therefore, we introduce a scaling term to perform scaled convolution in CKGConv. } We consider a kernel function $\psi: \RR^r \to \RR$. 
For a graph $\mathcal{G}=(\cV, \cE)$ with node-signal function $\chi: \cV \to \RR$, 
CKGConv is defined as:\footnote{\mlh{In the Euclidean domain, conventional {\em convolution} involves {\em reversal} and {\em shifting} of the filter (kernel). The meaning of reversal is not obvious in the graph domain. Although there is no explicit reversal in our procedure, the kernel $\psi$ is a mapping from a positive relative positional encoding $\mathbf{P}_{i,j}$. For each node $i$, we can thus view it as a symmetric filter with respect to a corresponding absolute positional encoding (that we do not specify), and reversal would not change the filtering coefficient for a node $j$. We therefore retain the usage of the terminology \emph{convolution} in this work.}}
\begin{equation}
    (\chi \star \psi)(i) := \frac{1}{|\operatorname{supp}_\psi(i)|}\sum_{j \in \operatorname{supp}_\psi(i)} \chi(j)\cdot \psi(\P_{i,j}) + b\,.\label{eq:ck-gconv}
\end{equation}
Here $b \in \RR$ is a learnable bias term; the set $\operatorname{supp}_\psi(i)$ is the predefined support of kernel $\psi$ for node $i$ (i.e., $|\operatorname{supp}_\psi(i)|$ denotes the kernel size); and $\P_{i,j} \in \RR^K$ is the relative positional encoding.

\mlh{Owing to the flexibility of CKGConv,} we can set $\operatorname{supp}_\psi(i)$ to be the $K$-hop
neighborhood\footnote{The $K$-hop neighborhood of node $i$ is the set of nodes whose shortest-path distance from node $i$ is smaller than or equal to $K$.} of node $i$, \mlh{with $K$ being an arbitrary positive integer.} Alternatively, we can choose the support to be the entire graph, thereby constructing a global kernel.
This flexibility arises because the
construction of pseudo-coordinates is decoupled from the evaluation of the convolution kernel.
We show that the globally supported variant is endowed with several desired theoretical properties (Sec.~\ref{sec:expressiveness} and Sec.~\ref{sec:eq_set_nn}).

\renewcommand{\mlh}[1]{#1}

\subsubsection{Depthwise Separable Convolution}

We extend the scalar-valued definition of CKGConv (shown in Eq.~\eqref{eq:ck-gconv}) to vector-valued signals ($\boldsymbol{\chi}: \cV \to \RR^d$ and $(\boldsymbol{\chi} \star \boldsymbol{\psi})(i): \cV \to \RR^{d'}$) via
the Depthwise Separable Convolution (DW-Conv) architecture~\cite{chollet2017XceptionDeepLearning},  
\begin{equation}
    (\boldsymbol{\chi} \star \boldsymbol{\psi})(i) :=  \mathbf{W}\big(\tfrac{1} {|\operatorname{supp}_\psi(i)|}\sum_{j \in \operatorname{supp}_\psi(i)} \boldsymbol{\chi}(j) \odot \boldsymbol{\psi}(\P_{i,j})\big) + \mathbf{b} \,.  
    \label{eq:ck-gconv-dw}
\end{equation}
Here $\boldsymbol{\psi}: \RR^K \to \RR^d$ is a kernel function acting on a vector;  $\mathbf{W} \in \RR^{d' \times d}$ and $\mathbf{b} \in \RR^{d'}$ are the learnable weights and bias, respectively, shared by all nodes; and $\odot$ stands for elementwise multiplication.

We can alternatively extend Eq.~\eqref{eq:ck-gconv} to multiple channels via grouped convolution~\cite{krizhevsky2012ImageNetClassificationDeep}, multi-head architectures~\cite{vaswani2017AttentionAllYou}, or even MLP-Mixer architectures~\cite{tolstikhin2021MLPMixerAllMLPArchitecture, touvron2023ResMLPFeedforwardNetworks}.
We select DWConv because it provides a favorable trade-off between expressiveness and the number of parameters.

\subsubsection{MLP-based Kernel Function}

In this work, as an example, we introduce kernel functions parameterized by multi-layer perceptrons (MLPs), but the proposed convolution methodology accommodates many other kernel functions.
Each MLP block consists of fully connected layers ($\func{FC}$), non-linear activation ($\sigma$), a normalization layer ($\func{norm}$) and a residual connection, 
inspired by ResNetv2~\cite{he2016IdentityMappingsDeep}:
\begin{equation}
   \func{MLP}(\mathbf{x}) := \mathbf{x} +  \func{FC} \circ \sigma \circ \func{Norm} \circ \func{FC} \circ \sigma \circ \func{Norm}(\mathbf{x})\,.
   \label{eq:mlp_block}
\end{equation}
Here $\circ$ denotes function composition; 
$\func{FC}(\mathbf{x}):= \mathbf{W}\mathbf{x}+\mathbf{b}$, with learnable weight matrix, $\mathbf{W} \in \RR^{r \times r}$, and bias, $\mathbf{b} \in \RR^r$; and we use GELU~\cite{hendrycks2023GaussianErrorLinear} as the default choice of $\sigma$.

The overall kernel function $\boldsymbol{\psi}$ is defined as:
\begin{equation}
\boldsymbol{\psi}(\mathbf{\P_{i,j}}):= \func{FC}\circ \func{Norm}\circ \func{MLP} \circ \dots \circ \func{MLP}(\P_{i,j})\,,
\end{equation}
where the last $\func{FC}: \RR^r \to \RR^{d}$ maps to the desired number of output channels.

\subsubsection{Degree Scaler}
As a known issue in graph learning, the scaled convolutions and mean-aggregations cannot properly preserve the degree information of nodes~\cite{xu2019HowPowerfulAre,corso2020PrincipalNeighbourhoodAggregation}. 
Therefore, we introduce a post-convolution adaptive degree-scaler into 
the node representation to recover such information,
following the approach proposed by~\citet{ma2023GraphInductiveBiases}:
\begin{equation}
   \begin{aligned}
    \bx_{i}' :=  \bx_{i} \odot \boldsymbol{\theta}_1 
    + \left(d_i^{1/2} \cdot \bx_{i} \odot \boldsymbol{\theta}_2 \right) \in \mathbb{R}^d \,.
   \end{aligned}
   \label{eq:deg_scaler}
\end{equation}
Here $d_i \in \RR$ is the degree of node $i$, and $\boldsymbol{\theta}_1, \boldsymbol{\theta}_2 \in \RR^r$ are learnable weight vectors.

As an alternative, we also introduce a variant that injects the degree information directly into the RRWP, $\P_{i,j} \in \RR^r$, before applying the kernel function 
$\psi$:
\begin{equation}
    \hat{\P}_{i,j} \! := \! \P_{i,j} \odot \, \boldsymbol{\theta}_1 
    + \left( (d_i^{1/2} \odot \boldsymbol{\theta}_2) \odot \P_{i,j} \odot (d_j^{-1/2} \odot \boldsymbol{\theta}_3)  \right)\,, 
\end{equation}
where $\boldsymbol{\theta}_1, \boldsymbol{\theta}_2, \boldsymbol{\theta}_3 \in \RR^r$ are learnable parameters and $\hat{\P}_{i,j}$ is used instead of $\P_{i,j}$ in other parts.
This variant enjoys several desired theoretical properties as discussed in Sec.~\ref{sec:poly_spectral}.
However, in practice, we did not observe any significant differences in empirical performance, and use Eq.~\eqref{eq:deg_scaler} in our experiments due to its computational efficiency.

\subsubsection{Overall Architecture of CKGCN}

The overall multi-layer architecture of the proposed model, denoted by Continuous Kernel Graph Convolution Network~(CKGCN), 
consists of $L$ CKGConv-blocks as the backbone, together with task-dependent output heads (as shown in Fig.~\ref{fig:arch} in Appendix~\ref{appx:arch}). 
Each CKGConv block consists of a CKGConv layer and a feed-forward network (FFN), with residual connections and a normalization layer:
\begin{equation}
    \func{CKGConvBlock}(\cdot)\!:=\! \func{norm}\hspace{0.05em}\circ \hspace{0.05em}\func{FFN}\hspace{0.05em}\circ\hspace{0.05em} \func{norm} \hspace{0.05em}\circ\hspace{0.05em} \func{CKGConv} (\cdot)\,. 
\end{equation}
We use BatchNorm~\cite{ioffe2015BatchNormalizationAccelerating} in the main branch as well as in the kernel functions. Using LayerNorm~\cite{ba2016LayerNormalization} has the potential to cancel out the degree information~\cite{ma2023GraphInductiveBiases}.
Appendix~\ref{appx:deg_and_LN} presents additional architectural details.
The input node/edge attributes ($\mathbf{x}'_i \in \RR^{d_h}$ and $\mathbf{e}'_{i,j} \in \RR^{d_e}$) and the absolute/relative positional encoding (RRWP) are concatenated: $\mathbf{x}_i  = [\mathbf{x}'_i \| \P'_{i,i}] \in \RR^{d_h + K}$ and $\P_{i,j} = [\mathbf{e}'_{i,j} \| \P'_{i,j}] \in \RR^{d_e + K}$, where $\P'_{i,j}$ denotes the input PE.
A linear projection (a {\em stem}) maps to the desired dimensions before the backbone. 
If the data does not include node/edge attributes, zero-padding is used. For a fair comparison, we use the same task-dependent output heads as previous work~\cite{rampasek2022RecipeGeneralPowerful}.

\subsection{Theory: CKGCN Is as Expressive as Graph Transformers}
\label{sec:gd-wl}
\label{sec:expressiveness}

\citet{zhang2023RethinkingExpressivePower} prove that Graph Transformers with generalized distance (GD) can be as powerful as GD-WL with a proper choice of attention mechanisms,
thus going beyond 1-WL and bounded by 3-WL.
We provide a similar constructive proof, demonstrating that CKGConv with GD is as powerful as GD-WL. 
It achieves the same theoretical expressiveness as SOTA graph transformers~\cite{zhang2023RethinkingExpressivePower, ma2023GraphInductiveBiases}, with respect to the GD-WL test.

\begin{proposition}\label{prop:ckgconv_gdwl}
A Continuous Kernel Graph Convolution Network~(CKGCN), stacking feed-forward networks (FFNs) and globally supported CKGConvs with generalized distance (GD) as pseudo-coordinates, 
is as powerful as GD-WL, when choosing the proper kernel $\boldsymbol{\psi}$.
\end{proposition}

The proof is provided in Appendix~\ref{appx:gd_wl}.

%% file: section/theory.tex
\section{Relationship with Previous Work}

\subsection{Beyond the Limitations of MPNNs}

Despite being widely used, MPNNs are known to exhibit certain limitations: (1) over-smoothing; (2) over-squashing and under-reaching; (3) expressive power limited to $1$-WL. 
By contrast, 
CKGConv inherently addresses these constraints. \newline
\textbf{\textit{Over-smoothing}}~\cite{li2018DeeperInsightsGraph, oono2020GraphNeuralNetworks} arises because most MPNNs apply a smoothing operator (a blurring kernels or low-pass filter).  
\mlh{CKGConv can generate sharpening kernels, and thus does not suffer from oversmoothing, as illustrated in a toy example in Appendix~\ref{sec:toy_oversmoothing}.}
\textbf{\textit{Over-squashing}}~\cite{alon2020BottleneckGraphNeural, topping2022UnderstandingOversquashingBottlenecks} is mainly due to the local message-passing within the one-hop neighborhood.
The kernels in CKGConv can have supports beyond one-hop neighborhoods.
Both the empirical performance on the Long-Range Graph Benchmark~\cite{dwivedi2022LongRangeGraph} shown in Table~\ref{tab:lrgb}
and the ablation study in Appendix~\ref{sec:ablation_khop} showcase the effect of expanding the supports and indicate the necessity to go beyond local message-passing.
\newline
Regarding the \textbf{\textit{expressiveness}}~\cite{xu2019HowPowerfulAre, morris2019WeisfeilerLemanGo, loukas2020WhatGraphNeural}, we have demonstrated in Sec.~\ref{sec:expressiveness} that CKGConv can reach expressive power equivalent to GD-WL, thus going beyond $1$-WL algorithms.
The empirical experiments also validate the capacity of CKGConv.
\color{black}

\subsection{Equivariant Set Neural Networks}
\label{sec:eq_set_nn}

In general, graphs can be viewed as a set of nodes with observed structures among them.
Here, we demonstrate that 
when the pseudo-coordinates do not encode any graph structure,
CKGConv can naturally degenerate to the general form of a layer in an \textit{Equivariant Set Network}~\cite{segol2020UniversalEquivariantSet}.
This matches the natural transition between graph data and set data. The following proposition states that when we use 1-RRWP (i.e., an Identity matrix) as the pseudo-coordinate (and thus ignore any graph structure),
CKGConv is equivalent to a layer of an equivariant set network. 

\begin{proposition}\label{prop:ckgconv_deepset}
With $1$-RRWP $\P=[\mathbf{I}]$, 
CKGConv with a globally supported kernel can degenerate to the following general form of a layer in an \textit{Equivariant Set Network} (Eq.~(8) in ~\citet{segol2020UniversalEquivariantSet}):
\begin{equation}
\begin{aligned}
  (\chi \star \psi)(i)   =&  \gamma \cdot \chi(i) + \beta \cdot \big(\frac{1}{|\cV|} \sum_{j \in \cV} \chi(j)\big) + b\,.
\end{aligned}
\end{equation}
Here $\gamma, \beta, b \in \RR$ are learnable parameters.
This can be directly generalized to vector-valued signals.
\end{proposition}

The proof is provided in Appendix~\ref{appx:set_nn}.

\subsection{Polynomial Spectral and Diffusion Enhanced GNNs}
\label{sec:poly_spectral}

\renewcommand{\L}{\tilde{\mathbf{L}}}
\newcommand{\sA}{\tilde{\mathbf{A}}}
\newcommand{\y}{\mathbf{y}}

\emph{Polynomial spectral GNNs}
\mlh{approximate spectral convolution} by fixed-order polynomial functions of the symmetric normalized graph Laplacian matrix $\L =\mathbf{I} - \D^{-1/2}\A\D^{-1/2} \in \mathbb{R}^{n \times n}$.
\mlh{Similarly,} diverse \emph{Diffusion Enhanced GNNs}
use polynomial parameterization with the diffusion operator $\sA$ or $\rw$ replacing the graph Laplacian. 

The following proposition states that CKGConv can represent any polynomial spectral GNN or diffusion-enhanced GNN of any order \mlh{with suitable injections of the node degrees.}

\begin{proposition}\label{prop:ckgconv_spectral}
With $K$-RRWP $\P_{i, j} \in \RR^K$ as pseudo-coordinates, 
CKGConv with a linear kernel $\psi$ can represent any \emph{Polynomial Spectral GNN} or any \emph{Diffusion Enhanced GNNs} of $(K{-}1)$th order exactly, regardless of the specific polynomial parameterization,
if degree $d_i^{1/2}$ and $d_j^{-1/2}$ are injected to $\P_{i,j}$ properly.
\end{proposition}
The proof is presented in Appendix~\ref{appx:poly_spectral}.
If $K \to \infty$, CKGConv can closely approximate a full spectral GNN. 
This also highlights the relationship between RRWP and graph spectral wavelets ~\cite{hammond2011WaveletsGraphsSpectral}.

Note that CKGConv is strictly \emph{more expressive} than previous polynomial spectral GNNs and diffusion-enhanced GNNs. 
\mlh{While polynomial spectral GNNs and diffusion-enhanced GNNs are constrained by the linear combinations of powers of Laplacian/diffusion operators,
CKGConv, equipped with non-linear kernels $\psi$ such as MLPs, can construct more general convolution kernels. Since MLPs are universal function approximators~\cite{hornik1989MultilayerFeedforwardNetworks}, CKGConv can represent a considerably richer class of functions.
The ablation study on the kernel functions (in Appendix~\ref{sec:ablation_kernel_func}) also verifies the importance of introducing kernel functions beyond linear transformations.}

\subsection{Fourier Features of Graphs}

As mentioned in Proposition~\ref{prop:ckgconv_spectral},
RRWP can be viewed as a set of bases for a polynomial vector space, which approximates the full Fourier basis of graphs (i.e., eigenvectors of the Laplacian).
Therefore, RRWP can be viewed as a set of (approximate) Fourier features under certain transformations.
Likewise in the Euclidean space, \citet{tancik2020FourierFeaturesLet} propose to construct \emph{Fourier features} from coordinates to let MLPs better capture high-frequency information.

Another existing approach broadly related to our work is the Specformer~\cite{bo2023SpecformerSpectralGraph}, which generates graph spectral filters via transformers, given a sampled collection of Fourier bases\footnote{Operating on the full Fourier bases has $O(N^3)$ computational complexity, where $N$ is the number of nodes in the graph.} in the spectral domain.
\mlh{Specformer approximates the full Fourier bases from the spectral perspective, whereas CKGConv performs an approximation in the spatial domain.}
In a similar fashion to the contrast between the Fourier transform and the wavelet decomposition, Specformer achieves better localization on frequencies, and CKGConv exhibits better localization spatially.

\subsection{Graph Transformers}

As shown in Sec.~\ref{sec:expressiveness}, with the same generalized distances (e.g., SPD, RD, RRWP) as relative positional encoding or pseudo-coordinates, 
CKGConv can reach the same theoretical expressive power as Graph Transformers,
with respect to graph isomorphism tests.

From a filtering perspective, self-attention in (Graph) Transformers can be viewed as a dynamic filter~\cite{park2021HowVisionTransformers}. However, the filter coefficients are constrained to be positive, and thus self-attention can only perform blurring or low-pass filtering.
In contrast, CKGConv is a non-dynamic filter, but has the flexibility to include positive and negative coefficients simultaneously and thus can generate sharpening kernels.

In this work, we do not claim that CKGConv is better than Graph Transformers, or vice versa.
We emphasize that each approach has its own advantages. 
The contrasting strengths of dynamic and sharpening,
present an intriguing possibility of developing architectures that combine the strengths of graph transformers and continuous convolution.
Exploratory experiments in Sec.~\ref{sec:study_ckgcn_gt_difference} highlight the behavioral differences between graph transformers and CKGCNs, and examine the performance of a preliminary, naive combination.

%% file: section/exp.tex
\section{Experimental Results}
\label{sec:experiments}

\newcommand{\first}[1]{\textcolor{SeaGreen}{\textbf{#1}}}
\newcommand{\second}[1]{\textcolor{BurntOrange}{\textbf{#1}}}
\newcommand{\third}[1]{\textcolor{Periwinkle}{\textbf{#1}}}

\begin{table*}[h!]
    \centering
    \caption{
    Test performance in five benchmarks from \cite{dwivedi2022BenchmarkingGraphNeural, ma2023GraphInductiveBiases, bo2023SpecformerSpectralGraph}. 
    Shown is the mean $\pm$ s.d. of 4 runs with different random seeds. Highlighted are the top \first{first}, \second{second}, and \third{third} results. 
    \# Param under $500K$ for ZINC, PATTERN, CLUSTER and $\sim 100K$ for MNIST and CIFAR10.}
    \vskip 0.15in
    
    \resizebox{0.96\textwidth}{!}{\footnotesize
    \begin{tabular}{lccccc}
    \toprule
       \textbf{Model}  &\textbf{ZINC} &\textbf{MNIST} &\textbf{CIFAR10} &\textbf{PATTERN} &\textbf{ CLUSTER} \\
       \cmidrule{2-6} 
       &\textbf{MAE}$\downarrow$  &\textbf{Accuracy}$\uparrow$ &\textbf{Accuracy}$\uparrow$ &\textbf{W. Accuracy}$\uparrow$ &\textbf{W. Accuracy}$\uparrow$ \\
       \midrule
       GCN  &$0.367\pm0.011$ &$90.705\pm0.218$ &$55.710\pm0.381$ &$71.892\pm0.334$ &$68.498\pm0.976$ \\
GIN  &$0.526\pm0.051$ &$96.485\pm0.252$ &$55.255\pm1.527$ &$85.387\pm0.136$ &$64.716\pm1.553$ \\
GAT &$0.384\pm0.007$ &$95.535\pm0.205$ &$64.223\pm0.455$ &$78.271\pm0.186$ &$70.587\pm0.447$ \\
GatedGCN &$0.282\pm0.015$ &$97.340\pm0.143$ &$67.312\pm0.311$ &$85.568\pm0.088$ &$73.840\pm0.326$ \\
GatedGCN-LSPE &$0.090\pm0.001$ &$-$ &$-$ &$-$ &$-$ \\
PNA &$0.188\pm0.004$ &$97.94\pm0.12$ &$70.35\pm0.63$ &$-$ &$-$ \\
GSN  &$0.101\pm0.010$ &$-$ &$-$ &$-$ &$-$ \\
\midrule
DGN &$0.168\pm0.003$ &$-$ &\second{${72.838\pm0.417}$} &$86.680\pm0.034$ &$-$ \\
Specformer &\second{$\mathbf{0.066\pm0.003}$} &- &- &- &- \\
\midrule CIN &{${0.079}\pm{0.006}$} &$-$ &$-$ &$-$ &$-$ \\
CRaW1 &$0.085\pm0.004$ &${97.944}\pm{0.050}$ &$69.013\pm0.259$ &$-$ &$-$ \\
GIN-AK+ &${0 . 0 8 0}\pm{0 . 0 0 1}$ &$-$ &$72.19\pm0.13$ &\third{$\mathbf{86.850\pm0.057}$} &$-$ \\
\midrule SAN &$0.139\pm0.006$ &$-$ &$-$ &$86.581\pm0.037$ &$76.691\pm0.65$ \\
Graphormer &$0.122\pm0.006$ &$-$ &$-$ &$-$ &$-$ \\
K-Subgraph SAT &$0.094\pm0.008$ &$-$ &$-$ &{${86.848\pm0.037}$} &$77.856\pm0.104$ \\
EGT &$0.108\pm0.009$ &\second{$\mathbf{98.173\pm0.087}$} &$68.702\pm0.409$ &$86.821\pm0.020$ &\second{$\mathbf{79.232\pm0.348}$} \\
Graphormer-URPE &$0.086\pm0.007$ &$-$ &$-$ &$-$ &$-$  \\
Graphormer-GD &$0.081\pm0.009$ &$-$ &$-$ &$-$ &$-$ \\
\midrule
 GPS &\third{$\mathbf{0.070}\pm\mathbf{0.004}$} &{${98.051\pm0.126}$} &{${72.298\pm0.356}$} &$86.685\pm0.059$ &{${78.016\pm0.180}$} \\
GRIT &\first{$\mathbf{0.059\pm0.002}$}&\third{$\mathbf{98.108\pm0.111}$} &\first{$\mathbf{76.468\pm0.881}$} &\second{$\mathbf{87.196\pm0.076}$} &\first{$\mathbf{80.026\pm0.277}$} \\
\midrule
CKGCN &\first{$\mathbf{0.059\pm0.003}$}       &\first{$\mathbf{98.423\pm0.155}$}         &\third{$\mathbf{72.785\pm0.436}$}        &\first{$\mathbf{88.661\pm0.143}$}  &\third{$\mathbf{79.003\pm0.140}$}       
\\
\bottomrule
\end{tabular}
}
\label{tab:exp_main}
\vskip -0.1in
\end{table*}


\subsection{Benchmarking CKGCN} 
We evaluate our proposed method on five datasets from {\em Benchmarking GNNs}~\cite{dwivedi2022BenchmarkingGraphNeural} and another two datasets from {\em Long-Range Graph Benchmark}~\cite{dwivedi2022LongRangeGraph}. 
These benchmarks include diverse node- and graph-level learning tasks such as node classification, graph classification, and graph regression. They test an algorithm's ability to focus on graph structure encoding, to perform node clustering, and to learn long-range dependencies. 
The statistics of these datasets and further details of the experimental setup are deferred to Appendix~\ref{appendix:experiment_details}.

\textbf{Baselines}
We compare our methods with
\begin{itemize}[leftmargin=*,noitemsep,topsep=0pt]
\item \emph{SOTA Graph Transformer:} GRIT~\cite{ma2023GraphInductiveBiases}; 
\item \emph{Hybrid Graph Transformer (MPNN+self-attention):} GraphGPS~\cite{rampasek2022RecipeGeneralPowerful};
\item \emph{Popular Message-passing Neural Networks:} GCN~\cite{kipf2017SemiSupervisedClassificationGraph}, GIN~~\cite{xu2019HowPowerfulAre} and its variant with edge-features~\cite{hu2020StrategiesPretrainingGraph}, GAT~\cite{velickovic2018GraphAttentionNetworks}, GatedGCN~\cite{bresson2018ResidualGatedGraph}, GatedGCN-LSPE~\cite{dwivedi2022GraphNeuralNetworks}, and
PNA~\cite{corso2020PrincipalNeighbourhoodAggregation}; 
\item \emph{Other Graph Transformers:} Graphormer~\cite{ying2021TransformersReallyPerform}, 
K-Subgraph SAT~\cite{chen2022StructureAwareTransformerGraph}, EGT~\cite{hussain2022GlobalSelfAttentionReplacement}, SAN~\cite{kreuzer2021RethinkingGraphTransformers}, 
Graphormer-URPE~\cite{luo2022YourTransformerMay}, and
Graphormer-GD~\cite{zhang2023RethinkingExpressivePower};
\item \emph{SOTA Spectral Graph Neural Networks:} Specformer~\cite{bo2023SpecformerSpectralGraph} and
DGN~\cite{beani2021DirectionalGraphNetworks}; and
\item \emph{Other SOTA Graph Neural Networks:} 
GSN~\cite{bouritsas2022ImprovingGraphNeural},
CIN~\cite{bodnar2021WeisfeilerLehmanGo},
CRaW1~\cite{tonshoff2023WalkingOutWeisfeiler}, and
GIN-AK+~\cite{zhao2022StarsSubgraphsUplifting}.
\end{itemize}

\textbf{Benchmarks from Benchmarking GNNs} In Table~\ref{tab:exp_main}, we report the results on five datasets from~\cite{dwivedi2022BenchmarkingGraphNeural}: ZINC, MNIST, CIFAR10, PATTERN, and CLUSTER. 
We observe that the proposed CKGConv achieves the best performance for 3 out of 5 datasets and is ranked within the three top-performing models for the other 2 datasets.
Compared to the hybrid transformer, \emph{GraphGPS}, consisting of an MPNN and self-attention modules (SAs), CKGCN outperforms on all five datasets, indicating the advantage of the continuous convolution employed by CKGConv over MPNNs, even when enhanced by self-attention.
\emph{GRIT} and CKGConv achieve comparable performance  but exhibit advantage in different datasets. CKGConv outperforms on MNIST and PATTERN while GRIT performs better on CIFAR10 and CLUSTER.
This suggests that the capability to learn \textit{dynamic} kernels and \textit{sharpening} kernels might have different impact and value on the empirical performance depending on the nature of the dataset.
Notably, CKGConv exhibits superior performance compared to all other convolution-based GNNs for four of the five datasets. The only exception is CIFAR10, where CKGConv is slightly worse than DGN,  although the difference is not statistically significant.
\footnote{According to a two-sided t-test at the 5\% significance level.}

\textbf{Long-Range Graph Benchmark (LRGB)} Graph transformers demonstrate advantages over MPNNs in modeling long-range dependencies.
Here, we verify the capacity of CKGConv to model long-range dependencies.
We conduct experiments on two peptide graph datasets from the Long-Range Graph Benchmark (LRGB)~\cite{dwivedi2022LongRangeGraph}.
The obtained results are summarized in Table~\ref{tab:lrgb}.
On both datasets, CKGConv obtains the second-best mean performance.
Based on a two-sample one-tailed t-test, the performance is not significantly different from the best-performing algorithm (GRIT). There is, however, a statistically significant difference between CKGConv's performance and the third-best algorithm's performance for both datasets.
This demonstrates that our model is able to learn long-range interactions, on par with the SOTA graph transformers.

\begin{table}[tb!]
    \caption{Test performance on two benchmarks from long-range graph benchmarks (LRGB)~\cite{dwivedi2022LongRangeGraph}. 
    Shown is the mean $\pm$ s.d. of 4 runs with different random seeds. Highlighted are the top \first{first}, \second{second}, and \third{third} results.
    \#~Param $\sim 500K$.
    }
    \vskip 0.15in
    \setlength{\tabcolsep}{2pt}
    \centering
    \resizebox{0.48\textwidth}{!}{\footnotesize
    \begin{tabular}{lcc}
    \toprule
       \textbf{Model}  &\textbf{Peptides-func} &\textbf{Peptides-struct} \\
       \cmidrule{2-3} 
       &\textbf{AP}$\uparrow$ &\textbf{MAE}$\downarrow$ \\
       \midrule
       GCN  &$0.5930\pm0.0023$  &$0.3496\pm0.0013$ \\
GINE &$0.5498\pm0.0079$ &$0.3547\pm0.0045$ \\
GatedGCN &$0.5864\pm0.0035$ &$0.3420\pm0.0013$ \\
GatedGCN+RWSE &$0.6069\pm0.0035$ &$0.3357\pm0.0006$ \\
\midrule 
Transformer+LapPE &$0.6326\pm0.0126$ &$0.2529\pm0.0016$ \\
SAN+LapPE &$0.6384\pm0.0121$  &$0.2683\pm0.0043$ \\
SAN+RWSE&{${0.6439\pm0.0075}$}  &$0.2545\pm0.0012$ \\
\midrule
GPS &\third{$\mathbf{0.6535\pm0.0041}$} &\third{$\mathbf{0.2500\pm0.0012}$}\\
GRIT &\first{$\mathbf{0.6988\pm0.0082}$} &\first{$\mathbf{0.2460\pm0.0012}$}
\\
\midrule
CKGCN &\second{$\mathbf{0.6952\pm0.0068}$} &\second{$\mathbf{0.2477\pm0.0018}$} \\
       \bottomrule
    \end{tabular}
    }
    \label{tab:lrgb}
    \vskip -0.in
\end{table}


\subsection{The Flexible Kernels of CKGConv}

\mlh{
Convolutional kernels with both negative and positive coefficients have a long history.
Such kernels are widely used to amplify the signal differences among data points, e.g., signal sharpening and edge detection in image processing.
Here, we highlight that CKGConv has the flexibility to generate kernels that include 
negative and positive coefficients.
}

\subsubsection{CKGConv Kernel Visualization}

We show that CKGConv can learn positive and negative kernel coefficients from the data, without being forced to generate negative kernel coefficients. 
Therefore, we visualize the learned kernels of CKGConv from real-world graph learning tasks.
Specifically, we visualize two selected learned kernels from the depthwise convolution of CKGConv for each of the two graphs from the ZINC datasets, as shown in Fig.~\ref{fig:kernel}.
\mlh{Several learned kernels in CKGConv indeed generate both positive and negative coefficients.}

\begin{figure}[htb!]
    \centering
    \begin{minipage}{0.6\textwidth}
    \includegraphics[height=2.5cm]{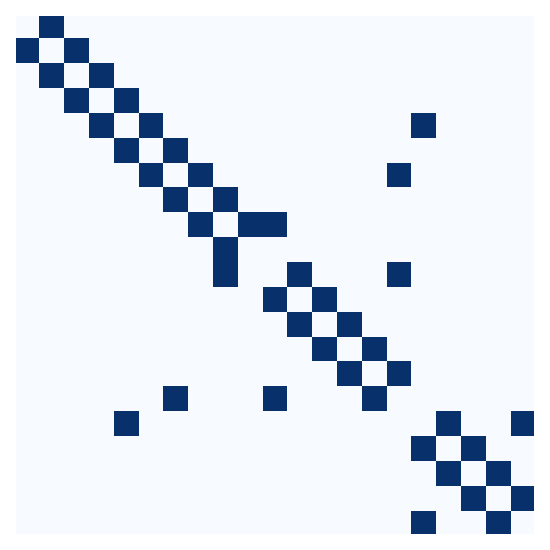}
    \includegraphics[height=2.5cm]{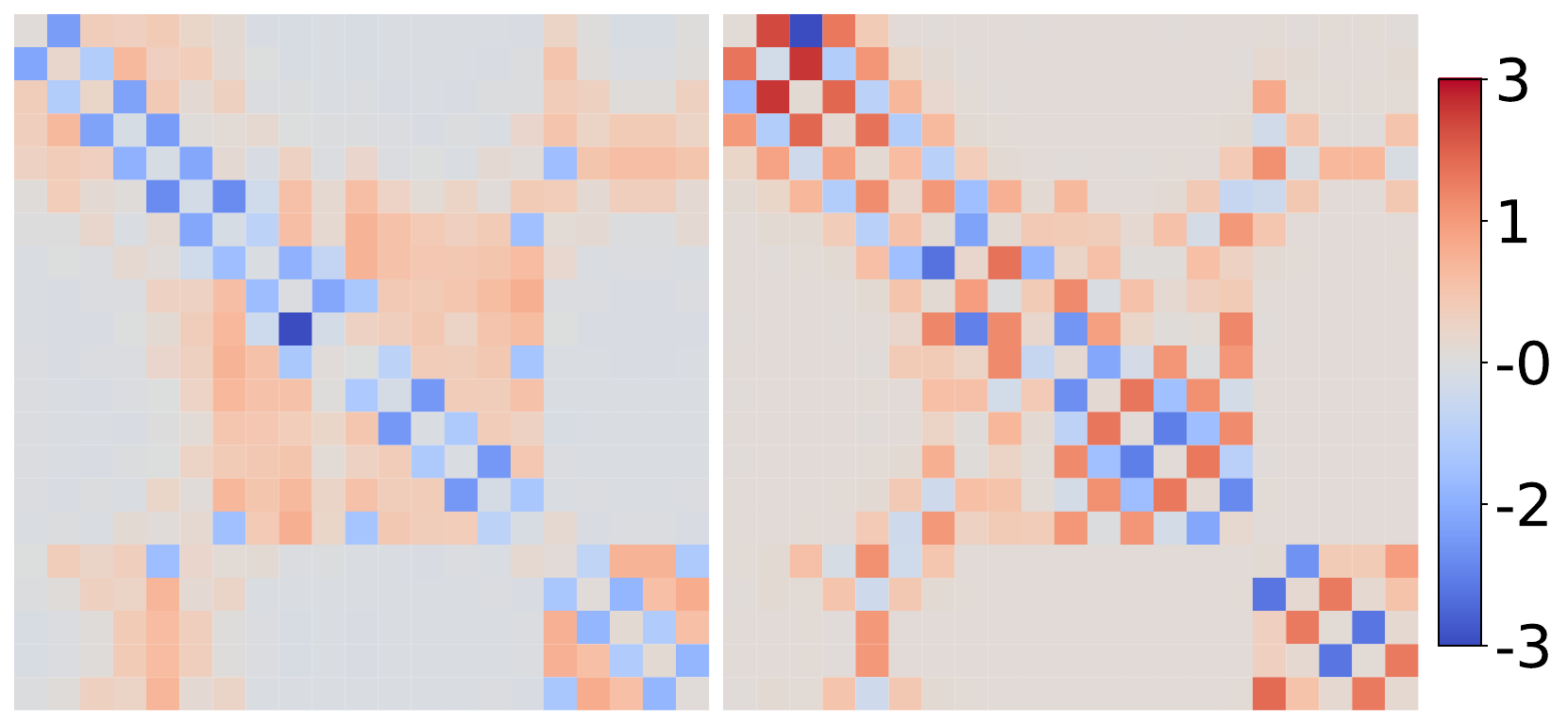} 
    \end{minipage} \hfill
    \begin{minipage}{0.6\textwidth}
    \includegraphics[height=2.5cm]{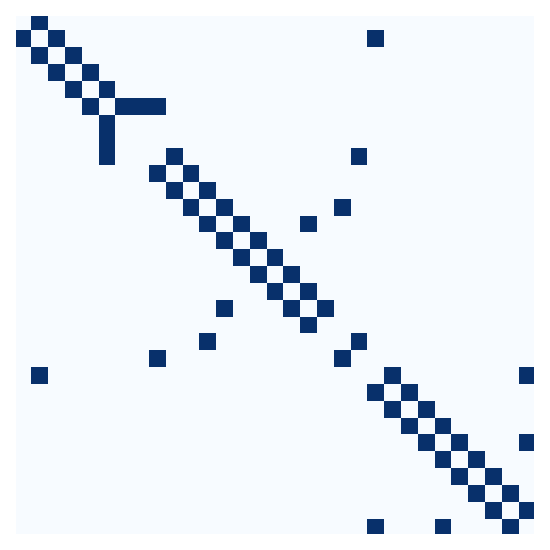}  
    \includegraphics[height=2.5cm]{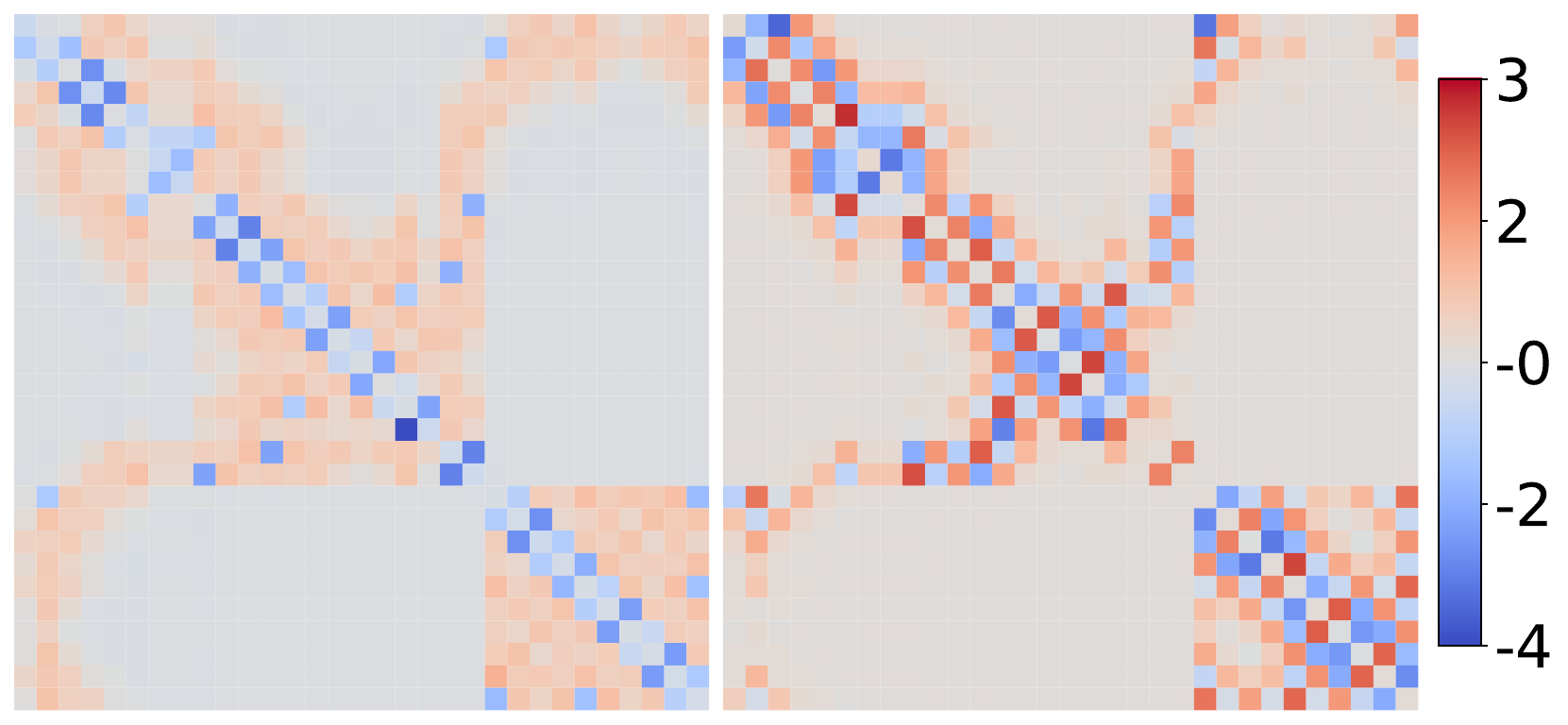}
    \end{minipage}
    \caption{Adjacency matrices and learned continuous kernels across multiple channels for two graphs from the ZINC dataset.}
    \label{fig:kernel}
    \vskip -0.1in
\end{figure}

\subsubsection{Ablation Study on Flexible Kernels}
\label{sec:ablation_flexible_kernel}

\mlh{
To showcase the importance of the flexible kernels in CKGConv on graph learning tasks, we conduct an ablation study on ZINC and PATTERN,
comparing CKGCN to its blurring-kernel variant,
which is constrained to generate all-positive coefficients by incorporating a Softmax operation.
}

\begin{table}[t!]
    \caption{Effect of different kernel types on ZINC and PATTERN.}
    \setlength{\tabcolsep}{2pt}
    \vskip 0.15in
    \centering
    \resizebox{0.48\textwidth}{!}{\footnotesize
    \begin{tabular}{l|cc|cc}
    \toprule
       \textbf{Model}  &\textbf{ZINC} &\textbf{PATTERN} &\multirow{2}{*}{\emph{Dynamic}} &\multirow{2}{*}{\emph{Flexible}} \\
       \cmidrule{2-3} 
       &\textbf{MAE}$\downarrow$ &\textbf{W. Acc.}$\uparrow$ \\
       \midrule
       GRIT  & \textbf{$0.059\pm0.002$}  &$87.196\pm0.076$ &\cmark &\xmark \\
       CKGCN   & \textbf{$0.059\pm0.003$} & \textbf{$88.661\pm0.143$} &\xmark &\cmark \\
       \hfill -Blurring &$0.073\pm0.003$ &$87.000\pm0.002$ &\xmark &\xmark \\
       \bottomrule
    \end{tabular}
    }
    \label{tab:kernel}
    \vskip -0.15in
\end{table}
\mlh{
From Table~\ref{tab:kernel}, 
constraining CKGCN to generate blurring kernels leads to remarkable performance deterioration.
In addition, GRIT with self-attention (SA) mechanisms, which can be viewed as {\em dynamic} blurring kernels~\cite{park2021HowVisionTransformers}, outperforms CKGCN-Blurring.
This observation indicates that both dynamic and flexible properties are beneficial to graph learning and hints at the potential combination of CKGConvs and SAs for graph learning.\footnote{Note that attention mechanisms typically require the incorporation of Softmax to stabilize the attention scores.}
}

\subsection{Sensitivity Study on the Choice of Graph PEs}
\label{appx:pe_study}

\mlh{
In this paper, we demonstrate the efficacy of CKGConv with RRWP to avoid PE becoming the bottleneck of the model performance.
However, CKGConv is not constrained to working with a specific graph PE. 
As depicted in Proposition~\ref{appx:gd_wl}, the choices of PE affect the expressive power of CKGConv,
depending on the structural/positional information encoded.
Therefore, in this section, we study the impact of using different graph PEs,
and demonstrate that CKConv can reach a competitive performance with other well-designed and expressive PEs besides RRWP.
}


We conduct the sensitivity study on ZINC datasets with four typical graph PEs: RRWP~\cite{ma2023GraphInductiveBiases}, 
Resistance Distance (RD)~\cite{zhang2023RethinkingExpressivePower}, 
Shortest-path distance (SPD)~\cite{ying2021TransformersReallyPerform}, 
and ``Pair-RWSE", which is constructed as relative PE by concatenating the Random Walk Structural Encoding (RWSE)~\cite{dwivedi2022GraphNeuralNetworks} for each node-pair.
\mlh{We add RWSE as the absolute PE to the node attribute when using other PEs, mimicking RRWP.}
The experimental setup follows the main experiment and the results of 4 runs are reported in Table ~\ref{tab:sensitivity_PE}.

\begin{table}[!h]
    \vskip -0.15in
    \centering
    \caption{The sensitive study of CKGCN on the choices of graph PEs. Shown is the mean $\pm$ s.d. of 4 runs.
    }
    \label{tab:sensitivity_PE}
    \vskip 0.15in
        \resizebox{0.45\textwidth}{!}{
    \begin{tabular}{l|llll}
    \toprule
        CKGCN & RRWP & RD & SPD & Pair-RWSE \\ \midrule
        \multirow{2}{*}{MAE $\downarrow$} & 0.059 & 0.062  & 0.072  & 0.081  \\
        & $\pm$ 0.003  & $\pm$ 0.004 & $\pm$ 0.003 & $\pm$ 0.002
        \\ \bottomrule
    \end{tabular}}
    \vskip -0.05in
\end{table}

The results of SPD and Pair-RWSE show that a sub-optimal PE design leads to worse performance of CKGCN.
However, with an expressive PE, 
CKGCN demonstrates stable performance: CKGCN with either RD or RRWP achieves competitive performance that is statistically indistinguishable from the state-of-the-art.\footnote{According to a two-sided t-test at the 5\% significance level.}

\subsection{CKGCN and GTs Behave Differently}
\label{sec:ensmb}
\label{sec:study_ckgcn_gt_difference}




\mlh{
Motivated by the observation in Sec.~\ref{sec:ablation_flexible_kernel} and previous work on ViT~\cite{park2021HowVisionTransformers},
in this section, we aim to demonstrate that CKGCNs and graph transformers learn complementary features in graph learning.
}
Thus, we conduct an ensembling experiment on ZINC to examine the effects of naively combining a SOTA graph transformer GRIT with CKGCN.

In Table~\ref{tab:ensmb}, we report the mean and standard deviation of MAE (employing bootstrapping) of an ensemble of GRIT models, an ensemble of CKGConv models, and a mixed ensemble using both of these models.
In each case, the total ensemble size is 4 (two of each model for the mixed ensemble). 
We observe that constituting the ensemble using both CKGConv and SAs offers a statistically significant advantage compared to 
either homogeneous ensemble. 
\begin{table}[h!]
    \vskip -0.15in
    \caption{Effect of ensembling on ZINC.}
    \label{tab:ensmb}
    \setlength{\tabcolsep}{2pt}
    \vskip 0.15in
    \centering
    {\footnotesize
    \begin{tabular}{l|ccc}
    \toprule
       Model
       &\textbf{GRIT-Ens.} &\textbf{CKGConv-Ens.} &\textbf{Mixed-Ens.} \\ 
       \midrule
       MAE $\downarrow$   &0.054$\pm$0.001 &0.054$\pm$0.002  &\textbf{0.051$\pm$0.001}* \\
       \bottomrule
    \end{tabular}
    }
        \vskip -0.05in
\end{table}

Based on the observation, we hypothesize that both CKGConvs and SAs have their own merits,
and it can be further advantageous to suitably combine them in the model architecture. 
Similar efforts have been undertaken in computer vision~\cite{park2021HowVisionTransformers, xiao2021EarlyConvolutionsHelp}.

\subsection{Further Analyses: Anti-Oversmoothing, Edge-detection, Support Sizes and Kernel Functions}

\mlh{
We include the results of additional experiments to further analyse the performance and behavior of CKGConv in Appendix~\ref{appx:additional_exp}. Specifically, we include:
\begin{itemize}[leftmargin=*,noitemsep,topsep=0pt]
    \item Two toy examples that demonstrate the advantages of (positive and) negative kernel coefficients.
    \begin{itemize}[leftmargin=*,noitemsep,topsep=0pt]
    \item Appendix~\ref{sec:toy_oversmoothing} showcases that CKGConv effectively counters \emph{oversmoothing}.
    \item Appendix~\ref{sec:toy_ridge_detect} demonstrates the efficacy of CKGConv for \emph{edge-detection}.\footnote{\emph{Edge-detection} refers to detecting signal discontinuities in signal processing.} 
\end{itemize}
    \item Two ablation/sensitivity studies on the kernel designs.
\begin{itemize}[leftmargin=*,noitemsep,topsep=0pt]
    \item Appendix~\ref{sec:ablation_khop} studies the impact of the \emph{support size}, which demonstrates the utility of localized kernels and highlights the importance of going \emph{beyond local message-passing}.
    \item Appendix~\ref{sec:ablation_kernel_func} analyzes the impact from the \emph{number of MLP blocks} in the kernel functions, which indicates the necessity of \emph{non-linear kernel functions}.
\end{itemize}
\end{itemize}
}


%% file: section/conclusion.tex
\section{Limitations}
On the computational side, 
a naive implementation of CKGConv with global support has $O(|\mathcal{V}|^2)$ complexity, the same as graph transformers.
A more efficient alternative implementation is provided in Appendix~\ref{appx:eff_impl}, which might prevent the usage of some operators such as BatchNorm.
The localized CKGConv can benefit from lower computation complexity but with weaker theoretical expressiveness.

\section{Conclusion}

Motivated by the lack of a flexible and powerful convolution mechanism for graphs,
we propose a general graph convolution framework, CKGConv, by generalizing continuous kernels to graph domains. These can recover most non-dynamic convolutional GNNs, from spatial to (polynomial) spectral.
Addressing the fundamentally different characteristics of graph domains from Euclidean domains,
we propose three theoretically and empirically motivated design innovations to accomplish the generalization to graphs.
Theoretically, we demonstrate that CKGConv possesses equivalent expressive power to SOTA Graph Transformers in terms of distinguishing non-isomorphic graphs via the GD-WL  test~\cite{zhang2023RethinkingExpressivePower}. We also provide 
theoretical connections to previous convolutional GNNs.
Empirically, the proposed CKGConv architecture either surpasses or achieves performance comparable to the SOTA across a wide range of graph datasets. It outperforms all other convolutional GNNs and achieves performance comparable to SOTA Graph Transformers. 
A further exploratory experiment suggests that CKGConv can learn non-dynamic sharpening kernels and extracts information complementary to that learned by the self-attention modules of Graph Transformers. This motivates a potential novel avenue of combining CKGConv and SAs in a single architecture.
\mlh{
Furthermore, the success of CKGConv motivates the generalization of continuous kernel convolutions to other non-Euclidean geometric spaces based on pseudo-coordinate designs.
}

\clearpage

%% file: section/appendix.tex
\section{Model Architecture and Implementation Details}
\label{appx:arch}

\subsection{Model Architecture}

In order to combine all the building blocks into one clear visualization, we provide an overview of the CKGCN in Figure~\ref{fig:arch}.
\vspace{-0.5cm}

\begin{figure*}[h]
    \centering
    \subfigure[]{
    \centering
    \begin{minipage}[b]{.6\linewidth}
    \includegraphics[width=\textwidth]{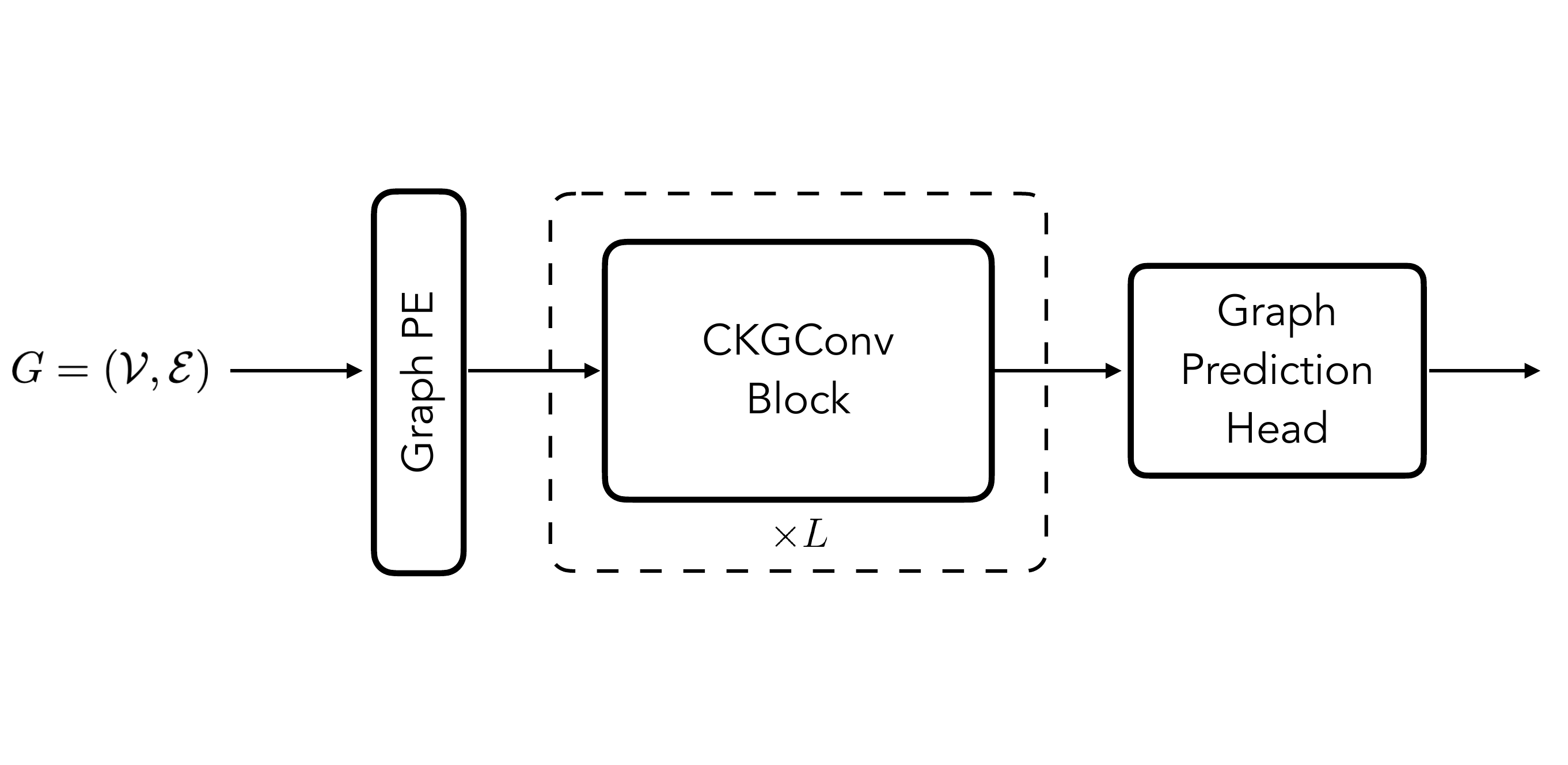}
    \vspace{0.05in}
    \end{minipage}
    }
    \subfigure[]{
    \centering
    \begin{minipage}[b]{.2\linewidth}
    \includegraphics[width=\textwidth]{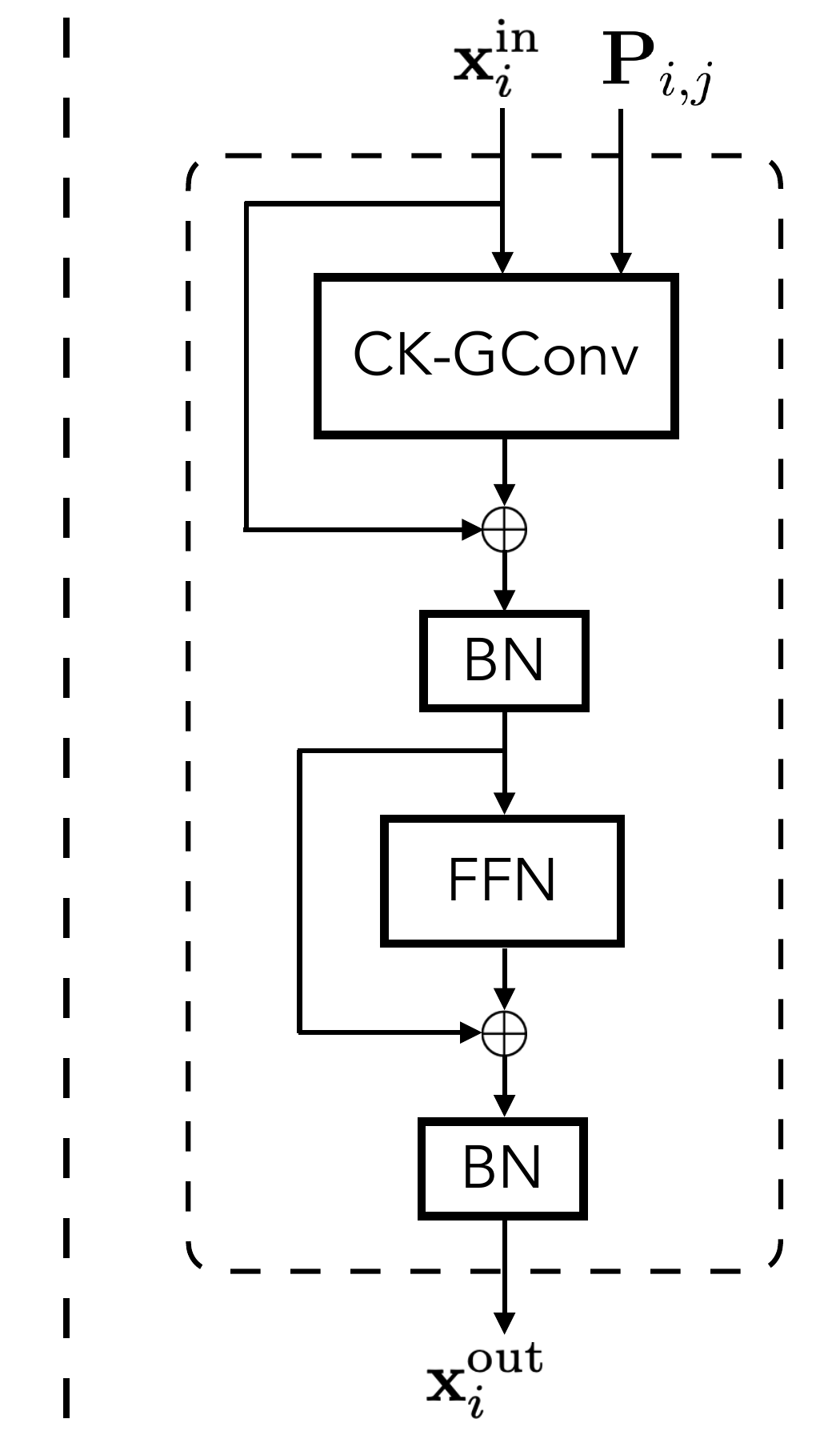}
    \end{minipage}
    }
    \caption{(a) Detailed Architecture of CKGCN with $L$ CKGConv blocks and task-dependent output head, (b) the detailed design of each CKGConv-block.}
    \label{fig:arch}
\end{figure*}

\subsection{Rescaling of RRWP}
\label{appx:re-normalization}


The expected values of random walk based graph PEs, e.g., RRWP, 
are dependent on the graph orders.
For a graph with $N$ nodes, RRWP has the property that $\sum_{j \in \cV}\P_{i,j} = \mathbf{1} \Rightarrow \mathbb{E}_{j \sim \cV}[\P_{i,j}] = \mathbf{1}/N$.
Empirically, we found that removing this dependency is beneficial to CKGConv.
Therefore, we introduce an extra re-scaling 
 for RRWP by setting $\P_{i,j} \leftarrow N \cdot \P_{i,j}$.
For other graph PEs without such dependencies, e.g,. RD and SPD, 
this re-scaling is not necessary.

Following the approaches in GraphGPS~\cite{rampasek2022RecipeGeneralPowerful},
we introduce an extra BatchNorm~\cite{ioffe2015BatchNormalizationAccelerating} on the input RRWP to further normalize the input values.

\color{black}

\section{Additional Related Work}

\paragraph{Graph Positional Encoding}
\label{appx:graph_pe}
In recent years, positional and/or structural encoding has been widely studied
to enhance the performance of MPNNs~\cite{you2019PositionawareGraphNeural,ma2021GraphAttentionNetworks,li2020DistanceEncodingDesign,zhang2021EigenGNNGraphStructure,loukas2020WhatGraphNeural,dwivedi2022GraphNeuralNetworks, lim2023SignBasisInvariant, wang2022EquivariantStablePositional, you2021IdentityawareGraphNeural, velingker2023AffinityAwareGraphNetworks, 
bouritsas2022ImprovingGraphNeural}.
Due to the inherent properties of attention mechanisms~\cite{vaswani2017AttentionAllYou, lee2019SetTransformerFramework},
Graph Transformers rely on positional/structural encoding even more excessively.
Disparate designs have been proposed by previous works, from absolute ones~\cite{dwivedi2021GeneralizationTransformerNetworks, kreuzer2021RethinkingGraphTransformers, kim2022PureTransformersAre} to relative ones~\cite{ying2021TransformersReallyPerform, zhang2023RethinkingExpressivePower, ma2023GraphInductiveBiases, mialon2021GraphiTEncodingGraph, hussain2022GlobalSelfAttentionReplacement, park2022GRPERelativePositional}.
A recent work~\cite{zhou2023FacilitatingGraphNeural} has also explored the potential of computing positional encoding on higher-order simplicial complexes instead of on nodes. 
Positional encodings prioritize distance/affinity measures and structural encodings focus on structural patterns, but most encodings incorporate
both positional and structural information~\cite{srinivasan2020EquivalencePositionalNode}.

\section{Experimental Details}
\label{appendix:experiment_details}

\subsection{Description of Datasets} 

Table~\ref{tab:dataset} provides a summary of the statistics and characteristics of datasets used in this paper. The first five datasets are from \citet{dwivedi2022BenchmarkingGraphNeural}, and the last two are from \citet{dwivedi2022LongRangeGraph}.
Readers are referred to \citet{rampasek2022RecipeGeneralPowerful} for more details about the datasets.

\begin{table}[h!]
    \centering
    \caption{Overview of the graph learning datasets involved in this work~\cite{dwivedi2022BenchmarkingGraphNeural, dwivedi2022LongRangeGraph, irwin2012ZINCFreeTool}.}
    \vskip 0.15in
    \small
    \setlength{\tabcolsep}{1.6pt}
    \resizebox{1\textwidth}{!}{
    \begin{tabular}{l|ccccccc}
    \toprule
       \textbf{Dataset} &\textbf{\# Graphs} &\textbf{Avg. \# nodes} &\textbf{Avg. \# edges}  &\textbf{Directed} 
 &\textbf{Prediction level} &\textbf{Prediction task} &\textbf{Metric}\\
 \midrule
        ZINC &12,000  &23.2 &24.9 &No &graph &regression &Mean Abs. Error \\
        MNIST &70,000  &70.6  &564.5 &Yes  &graph &10-class classif. &Accuracy \\
        CIFAR10 &60,000 &117.6 &941.1 &Yes  &graph &10-class classif. &Accuracy \\
        PATTERN &14,000 &118.9 &3,039.3  &No &inductive node &binary classif. &Weighted Accuracy \\
        CLUSTER &12,000 &117.2 &2,150.9 &No &inductive node &6-class classif. &Weighted Accuracy \\ 
        \midrule
        Peptides-func &15,535 &150.9 &307.3 &No &graph &10-task classif. &Avg. Precision \\
        Peptides-struct &15,535  &150.9 &307.3 &No &graph &11-task regression &Mean Abs. Error  \\
        \bottomrule
    \end{tabular}
    }
    \label{tab:dataset}
\end{table}

\subsection{Dataset splits and random seed}
We conduct the experiments on the standard train/validation/test splits of the evaluated benchmarks, following previous works~\cite{rampasek2022RecipeGeneralPowerful, ma2023GraphInductiveBiases}.
For each dataset, we execute 4 runs with different random seeds (0,1,2,3) and report the mean performance and standard deviation.

\subsection{Optimizer and Learning Rate Scheduler}

We use AdamW~\cite{loshchilov2018DecoupledWeightDecay} as the optimizer and the Cosine Annealing Learning Rate scheduler~\cite{loshchilov2017SGDRStochasticGradient} with linear warm up.

\subsection{Hyperparameters}
Due to the limited time and computational resources, we did not perform an exhaustive search or a grid search for the hyperparameters.
We mostly follow the hyperparameter settings of GRIT~\cite{ma2023GraphInductiveBiases}, and make slight changes to adjust the number of parameters to match the commonly used parameter budgets.

We follow the most commonly used parameter budgets: up to 500k parameters for ZINC, PATTERN, CLUSTER, Peptides-func and Peptides-struct; and around 100k parameters for MNIST and CIFAR10.

The final hyperparameters are presented in Table~\ref{tab:bmgnn_hparam} and Table~\ref{tab:lrgb_hparam}.

\begin{table}[h!]
    \centering
    \caption{Hyperparameters for five datasets from BenchmarkingGNNs \cite{dwivedi2022BenchmarkingGraphNeural}. 
    }
    \vskip 0.15in
    \label{tab:bmgnn_hparam}
\begin{tabular}{lccccc}
\toprule
Hyperparameter &ZINC &MNIST &CIFAR10 &PATTERN &CLUSTER \\
\midrule
\# CKGConv-Block &10 &4 &3 &10 &16 \\
\quad - Hidden dim &64 &48 &56 &64 &54 \\
\quad - Dropout &0 &0 &0 &0 &$0.01$ \\
\quad - Norm. &BN &BN &BN &BN &BN \\
Graph pooling &sum &mean &mean &$-$ &$-$ \\
\midrule
PE dim ($K$-RRWP) &21 &18 &18 &21 &32 \\
\midrule
Kernel Func. \\
\quad - \# MLP Block &2 &2 &2 &2 &2 \\
\quad - Norm. &BN &BN &BN &BN  &BN \\  
\quad - Kernel dropout &$0.5$ &$0.5$ &$0.5$ &$0.5$ &$0.5$ \\
\quad - MLP dropout &$0.1$ &$0.2$ &$0.$ &$0.2$ &$0.5$ 
\\
\midrule
Batch size &32 &16 &16 &16 &16 \\
Learning Rate &$0.001$ &$0.001$ &$0.001$ &$0.001$ &$0.001$ \\
\# Epochs &2000 &200 &200 &200 &200 \\
\# Warmup epochs &50 &5 &5 &10 &10 \\
Weight decay &$1 \mathrm{e}-5$ &$1 \mathrm{e}-5$ &$1 \mathrm{e}-5$ &$1 \mathrm{e}-5$ &$1 \mathrm{e}-5$ \\
Min. lr. &$1 \mathrm{e}-6$ &$1 \mathrm{e}-4$  &$1 \mathrm{e}-4$ &$1 \mathrm{e}-4$ &$1 \mathrm{e}-4$ \\
\midrule
\# Parameters &433,663 &102,580 &105,320 &438,143 &499,754 \\
\bottomrule
\end{tabular}
\end{table}

\begin{table}[ht]
    \centering
        \caption{Hyperparameters for two datasets from the Long-range Graph Benchmark~\cite{dwivedi2022LongRangeGraph}.}
        \vskip 0.15in
    \label{tab:lrgb_hparam}
    \begin{tabular}{lcc}
\midrule Hyperparameter &Peptides-func &Peptides-struct \\
\midrule 
\# CKGConv-Block  &4 &4 \\
\quad - Hidden dim &96 &96  \\
\quad - Dropout &0 & 0.05 \\
\quad - Norm. &BN  &BN \\
Graph pooling &mean &mean \\
\midrule
PE dim ($K$-RRWP) & 24 & 24 \\
\midrule
Kernel Func. \\
\quad - \# MLP Block &2 &2 \\
\quad - Norm. & BN & BN \\
\quad - Kernel dropout & 0.5  &0.2 \\
\quad - MLP dropout & 0.2 &  0.2 \\
\midrule
Batch size &16 & 16 \\
Learning Rate & 0.001 &0.001 \\
\# Epochs &200 &200 \\
\# Warmup epochs &5 &5  \\
Weight decay & 0 & 0  \\
Min. lr. & 1e-4 & 1e-4  \\
\midrule
\# Parameters & 421,468 & 412,253  \\
\bottomrule
\end{tabular}
\end{table}

\subsection{Significance Test}

We conduct a two-sample one-tailed t-test to verify the statistical significance of the difference in performance.
The baselines' results are taken from \cite{ma2023GraphInductiveBiases}.

The statistical tests are conducted using the tools available at \url{https://www.statskingdom.com/140MeanT2eq.html}.

\subsection{Runtime}

We provide the runtime and GPU memory consumption of CKGCN in comparison to GRIT on ZINC as reference (Table~\ref{tab:runtime_memory}).
The timing is conducted on a single NVIDIA V100 GPU (Cuda 11.8) and 20 threads
of Intel(R) Xeon(R) Gold 6140 CPU @ 2.30GHz.

\begin{table}[b]
    \centering
    \caption{Runtime and GPU memory for GRIT~\cite{ma2023GraphInductiveBiases} and  
    CKGCN (Ours) on ZINC with batch size $32$. 
    The timing is conducted on a single NVIDIA V100 GPU (Cuda 11.8) 
    and 20 threads of Intel(R) Xeon(R) Gold 6140 CPU @ 2.30GHz.}
    \vskip 0.1in
    \begin{tabular}{l|lll}
    \toprule
        ZINC & CKGConv & GRIT \\ \midrule
        GPU Memory &  2146 MB & 1896 MB \\
        Training time  & 35.9 sec/epoch & 39.7 sec/epoch \\ \bottomrule
    \end{tabular}
    \label{tab:runtime_memory}
\end{table}

\subsection{Efficient Implementation}
\label{appx:eff_impl}

A more efficient implementation is achievable for CKGConv with global support when the graph order is large. Based on the following derivation, we can implement an algorithm with $O(|\mathcal{V}|S)$ complexity, where $S:=\mathbb{E}_{i \sim \cV}[|\{\P_{i,j} \neq \mathbf{0}\}|]$. The complexity thus depends on the order of the RRWP and the graph structure. 
. 

\newcommand{\cS}{\mathcal{S}}
Let $\mathcal{S}_i:=\{j \in \cV:\P_{i,j} \neq \mathbf{0}\}$,
then ignoring the bias term, Eq.~\eqref{eq:ck-gconv} can be written as
\begin{align}
    (\chi \star \psi)(i) 
    &=\frac{1}{|\cV|} \left( \sum_{j \in \cS_i} \chi(j) |\cV|\cdot \psi(\P_{i,j})  + \sum_{j' \in \cV \backslash \cS_i} \chi(j) \cdot \psi(\mathbf{0}) \right)\,, \\
    &=\frac{1}{|\cV|} \left( \sum_{j \in \cS_i} \chi(j) |\cV|\cdot \big(\psi(\P_{i,j}) - \psi(\mathbf{0})\big)\right)  
    + \frac{1}{|\cV|} \left( \sum_{j \in \cV} \chi(j) \cdot \psi(\mathbf{0}) \right) \,,\\
    &=\frac{1}{|\cV|} \left( \sum_{j \in \cS_i} \chi(j) |\cV|\cdot \big(\psi(\P_{i,j}) - \psi(\mathbf{0})\big)\right)  
    + \psi(\mathbf{0}) \cdot \frac{1}{|\cV|} \left( \sum_{j \in \cV} \chi(j) \right)\,. \label{eq:eff_ckgconv}
\end{align}
The second term of Eq.~\eqref{eq:eff_ckgconv} can be computed by global-average pooling of graphs shared by all nodes in $O(|\cV|)$, and the first term requires $O(|\cV|\cdot S)$ computation on average, where $S=\frac{1}{|\cV|} \sum_{i \in \cV}|\cS_i|$.

\section{Additional Experiments: Toy Examples, Sensitivity Study, and Ablation Study}
\label{appx:additional_exp}

\begin{figure}[h!]
    \centering
    \begin{minipage}[t]{0.3\textwidth}
            \centering
            Node Signals
            \vskip 0.1in
            \resizebox{0.75\textwidth}{!}{
    	\begin{tikzpicture}
		\node[state] (s0) {$0$};
		\node[state, right of=s0] (s1) {$1$};
		\node[state, right of=s1] (s3) {$0$};
		\node[state, below of=s0] (s2) {$1$};
		\node[state, right of=s2] (s4) {$0$};
		\node[state, below of=s3] (s5) {$1$};
		\draw (s0) edge[line width=2pt] node[above]{} (s1);
		\draw (s0) edge[line width=2pt] node[above]{} (s2);
		\draw (s1) edge[line width=2pt] node[above]{} (s3);
		\draw (s2) edge[line width=2pt] node[above]{} (s4);
		\draw (s3) edge[line width=2pt] node[above]{} (s5);
		\draw (s4) edge[line width=2pt] node[above]{} (s5);
	\end{tikzpicture}
    }
        \caption{The toy example for \textit{Anti-oversmoothing}.}
        \label{fig:toy_oversmooth}
    \end{minipage} 
    \hfill
    \begin{minipage}[t]{0.65\textwidth}
            \centering
        \begin{minipage}[t]{0.49\textwidth}
            \centering
            Node Signals
            \vskip 0.1in
            \resizebox{0.98\textwidth}{!}{
    	\begin{tikzpicture}
		\node[state] (s0) {$1$};
		\node[state, right of=s0] (s1) {$1$};
		\node[state, right of=s1] (s2) {$0$};
		\node[state, right of=s2] (s3) {$0$};
		\node[state, below of=s0] (s4) {$1$};
		\node[state, right of=s4] (s5) {$1$};
		\node[state, right of=s5] (s6) {$0$};
		\node[state, right of=s6] (s7) {$0$};
		\draw (s0) edge[line width=2pt] node[above]{} (s1);
		\draw (s1) edge[line width=2pt] node[above]{} (s2);
		\draw (s2) edge[line width=2pt] node[above]{} (s3);
		\draw (s0) edge[line width=2pt] node[above]{} (s4);
		\draw (s4) edge[line width=2pt] node[above]{} (s5);
		\draw (s5) edge[line width=2pt] node[above]{} (s6);
		\draw (s6) edge[line width=2pt] node[above]{} (s7);
		\draw (s3) edge[line width=2pt] node[above]{} (s7);
	\end{tikzpicture}
    }
        \end{minipage} \hfill
        \begin{minipage}[t]{0.49\textwidth}
            \centering
            Node Labels
            \vskip 0.125in
            \resizebox{0.98\textwidth}{!}{
    	\begin{tikzpicture}
		\node[state, draw=red!55, fill=red!20] (s0) {$0$};
		\node[state, draw=red!55, fill=red!20, right of=s0] (s1) {$1$};
		\node[state, draw=red!55, fill=red!20, right of=s1] (s2) {$1$};
		\node[state, draw=red!55, fill=red!20, right of=s2] (s3) {$0$};
		\node[state, draw=red!55, fill=red!20, below of=s0] (s4) {$0$};
		\node[state, draw=red!55, fill=red!20, right of=s4] (s5) {$1$};
		\node[state, draw=red!55, fill=red!20, right of=s5] (s6) {$1$};
		\node[state, draw=red!55, fill=red!20, right of=s6] (s7) {$0$};
		\draw (s0) edge[line width=2pt] node[above]{} (s1);
		\draw (s1) edge[line width=2pt] node[above]{} (s2);
		\draw (s2) edge[line width=2pt] node[above]{} (s3);
		\draw (s0) edge[line width=2pt] node[above]{} (s4);
		\draw (s4) edge[line width=2pt] node[above]{} (s5);
		\draw (s5) edge[line width=2pt] node[above]{} (s6);
		\draw (s6) edge[line width=2pt] node[above]{} (s7);
		\draw (s3) edge[line width=2pt] node[above]{} (s7);
	\end{tikzpicture}
    }
        \end{minipage}
        \vskip 0.1in
        \caption{The toy example for \textit{Edge Detection}.}
        \label{fig:toy_ridge_detect}
    \end{minipage}
    \label{fig:enter-label}
\end{figure}

\subsection{Toy Example: CKGConv Can Mitigate Oversmoothing}
\label{sec:toy_oversmoothing}

With the ability to generate both positive and negative coefficients, 
CKGConv can learn sharpening kernels (a.k.a. high-pass filters), 
which amplify the signal differences among data points to alleviate oversmoothing. 
Here, we provide a toy example to better illustrate CKGConv's capability to prevent oversmoothing.

We consider a simple graph with node signals as shown in Fig.~\ref{fig:toy_oversmooth},
and train 2-layer and 6-layer GCNs and CKGCNs, with 5-RRWP, to predict labels that match the node signals.
In this toy example, we remove all normalization layers, dropouts, as well as residual connections.
All models are trained for 200 epochs with the Adam optimizer (initial learning rate 1e-3) to overfit this binary classification task.
We report the results of 5 trials with different random seeds in Table~\ref{tab:anti-oversmoothing}.

\begin{table}[!ht]
    \centering
    \caption{Toy Example for Anti-oversmoothing (Fig.~\ref{fig:toy_oversmooth}): Training performance for reconstruction of node signals.
    Shown is the mean $\pm$ s.d. of 5 runs with different random seeds.}
    \vskip 0.15in
    \resizebox{0.8\textwidth}{!}{
    \begin{tabular}{l|llll}
    \toprule
        Train & 2-Layer GCN & 6-Layer GCN & 2-Layer CKGCN & 6-Layer CKGCN \\ \midrule
        BCE Loss & 0.071 $\pm$ 0.044 & 0.693 $\pm$ 2e-05 & 4e-05$\pm$ 2e-05 & 0.0  $\pm$ 0.0 \\ 
        Accuracy (\%) & 100 $\pm$ 0 & 50  $\pm$ 0 & 100  $\pm$  0 & 100  $\pm$ 0 \\ \bottomrule
    \end{tabular}
    }
    \label{tab:anti-oversmoothing}
\end{table}

As shown in the results, 
both the 2-layer GCN and 2-layer CKGCN can overfit the toy example and reach 100\% accuracy.
However, a 6-layer GCN fails to reconstruct the node signals.
Applying 6 smoothing convolutions (all-positive filter coefficients) in this small network leads to the aggregated representation at each node being very similar. 
This is a typical oversmoothing effect. The network predicts all nodes to have the same label, resulting in 50\% accuracy in the toy example.
In contrast, a 6-layer CKGCN not only reaches 100\% accuracy but also achieves a lower BCE loss, showcasing its strong capability in mitigating oversmoothing.

\subsection{Toy Example: CKGConv Can Do Edge-Detection} 
\label{sec:toy_ridge_detect}

Analogous to \emph{edge-detection} in signal processing, in the graph domain, kernels with positive and negative coefficients can be used to detect the nodes with cross-community-connection (the border nodes).
In such a setting, it is essential that the sign of the filter coefficient for the central node is opposite to those of the first-hop neighbors, in order to detect differences in attributes.

We introduce a toy example to demonstrate it as shown in Fig.~\ref{fig:toy_ridge_detect}:
given a graph with simple scalar node signals that match the ``community'' of a node (0 or 1),
the goal is to identify the border nodes with node labels as 1,

In this study, we consider single-channel convolution kernels.
We compare CKGConv with three all-positive kernels: 
GCNConv~\cite{kipf2017SemiSupervisedClassificationGraph}, CKGConv+Softmax (attention-like), and CKGConv+Softplus.
CKGConv and its variants use 5-RRWP with the hidden dimension of 5 in the kernel function. 
CKGCN+Softmax (sum-aggregation) and CKGCN+Softplus (mean-aggregation) apply Softmax and Softplus on the kernel coefficients, respectively, to constrain the kernels to have positive coefficients only.

We aim to verify the upper bounds for the expressivity of the convolutions by training them to overfit the task. 
Each convolution operator is trained for 200 epochs with the Adam optimizer (learning rate 1e-2) using binary cross entropy loss (BCE loss).
We report the last training BCE loss and accuracy from 5 trials with different initializations, in Table~\ref{tab:ridge_dectect}.

\begin{table}[!ht]
    \centering
    \caption{Toy Example for Edge-detection (Fig.~\ref{fig:toy_ridge_detect}): Training performance for reconstruction of node signals.
    Shown is the mean $\pm$ s.d. of 5 runs with different random seeds.}
    \vskip 0.15in
    \resizebox{0.8\textwidth}{!}{
    \begin{tabular}{l|llll}
    \toprule
        Train & CKGConv & GCNConv &  CKGConv+Softmax & CKGConv+Softplus \\ \midrule
        BCE Loss & \textbf{2e-4 $\pm$ 1e-05} & 0.693 $\pm$ 0.001 & 0.693 $\pm$ 0 & 0.687  $\pm$ 0.049 \\ 
        Accuracy (\%) & \textbf{100  $\pm$ 0} & 50 $\pm$ 0  & 50  $\pm$  0 & 60  $\pm$ 12.25 \\ 
        \bottomrule
    \end{tabular}}
    \label{tab:ridge_dectect} 
\end{table}

From the results, it is obvious that only convolution kernels with negative and positive values (regular CKConv) can reach 100\% accuracy and achieve a low BCE loss. 
All other convolution kernels, with only positive values, fail to identify the border nodes.

This toy example explains why negative coefficients are advantageous in graph learning tasks, which require the detection of signal differences among data points. 
Similar tasks include \emph{ridge-detection} and \emph{learning on heterophilic graphs}.

\color{black}


\subsection{Sensitivity study on kernel support sizes of CKGConv}
\label{sec:ablation_khop}

The CKGConv framework allows for kernels with pre-determined non-global supports,
analogous to the regular convolution in Euclidean spaces.
In this section, we study the effect of different pre-determined support sizes based on $K$-hop neighborhoods on ZINC datasets.
The sensitivity study follows the same experimental setup as the main experiments.
The timing is conducted on a single a single NVIDIA V100 GPU (Cuda 11.8) and 20 threads
of Intel(R) Xeon(R) Gold 6140 CPU @ 2.30GHz.

\newcommand{\pp}[1]{#1}

\begin{table}[h!]
    \centering
    \caption{The sensitivity study of the support sizes ($K$-hop neighborhoods) for CKGConv kernels with $21$-RRWP.}
    \vskip 0.15in
        \resizebox{0.6\textwidth}{!}{
    \begin{tabular}{l|c|c|c}
    \toprule
       ZINC  & 
       \pp{MAE $\downarrow$ } & Run-Time (sec/epoch) & GPU-Mem (MB) \\ \midrule
       1-hop  
       & \pp{0.073 $\pm$ 0.005} & 33.1    & 1186
         \\ 
       3-hop  &  
       \pp{0.063 $\pm$ 0.002} & 33.6 &  1522
        \\
       5-hop  & 
       \pp{0.061 $\pm$ 0.004} & 35.2 & 1624 \\
       11-hop  &
       \pp{0.063 $\pm$ 0.002} &  33.2 & 2128 \\
       21-hop  &
       \pp{0.060 $\pm$ 0.002} & 34.8 & 2148 \\
       Full & 0.059 $\pm$ 0.003 &   35.4 & 2148  \\
       \bottomrule
    \end{tabular}}
    \label{tab:k_hop_kernel}
\end{table}

From the results in Table~\ref{tab:k_hop_kernel}, with the same order of RRWP, 
larger support sizes usually lead to better 
empirical performances as well as greater GPU memory consumption.
On the one hand,
the results showcase the stability in performance of CKGConv, since
all CKGCN variants with $K{>}1$ hops reach competitive performance among the existing graph models, outperforming all existing GNNs and most Graph Transformers.
On the other hand, 
the results justify the necessity of introducing graph convolutions beyond the one-hop neighborhood (a.k.a., message-passing), since the one-hop CKGCN is significantly worse than the other variants with larger kernels. 
Furthermore, the sensitivity study also highlights the flexibility of CKGConv framework in balancing the computational cost and the capacity to model long-range dependencies,
by effortlessly controlling the kernel sizes like the Euclidean convolutions. 
Note that the kernel sizes are not necessarily tied with the order of RRWP or the counterparts of other graph PEs in CKGConv.

\subsection{Ablation study on the kernel functions.}
\label{sec:ablation_kernel_func}

As depicted in Sec.~\ref{sec:poly_spectral}, 
polynomial-based GNNs can be viewed as CKGCN with linear kernel functions.
However, allowing kernel functions with non-linearity is important,
since multilayer perceptions (MLPs) with non-linear activations can be universal function approximators~\cite{hornik1989MultilayerFeedforwardNetworks} while linear functions cannot.

To better understand the effects of the choices of kernel functions, we conduct an ablation study on ZINC and PATTERN datasets, following the experimental setup of the main experiments.
We compare different CKGCN variants, using kernel functions with $0, 1, \text{ and } 2$ MLP-blocks (as shown in Eq.~\eqref{eq:mlp_block}).
The width of each variant is adjusted to reach the parameter budget under $500$ K.
Note that $0$ MLP-blocks is equivalent to a linear kernel function and $2$ MLP-blocks setting is the default in CKGCN.

\begin{table}[h!]
    \centering
    \caption{The ablation study on \# MLP blocks in the kernel function of CKGConv.}
    \vskip 0.15in
    \resizebox{0.6\textwidth}{!}{
    \begin{tabular}{l|cc|cc}
        \toprule
       \# MLP &  \multicolumn{2}{c|}{ZINC} &   \multicolumn{2}{c}{PATTERN}\\ \cmidrule{2-3} \cmidrule{4-5}
       Blocks & MAE $\downarrow$ & \# param. & W.Accuracy $\uparrow$ & \# param.
       \\ \midrule
        0  & 0.074 $\pm$ 0.005  & 487 K & 87.355 $\pm$ 0.230 & 495 K \\ 
        1  &     0.065 $\pm$ 0.005  & 438 K & 88.955 $\pm$ 0.251 &  444 K     \\
        2   &  0.059 $\pm$ 0.003 & 434 K  &   88.661 $\pm$ 0.142 &   438 K  
         \\
         \bottomrule
    \end{tabular}
    }
    \label{tab:ablation_mlp_blocks}
\end{table}

From the results of the ablation study (as shown in Table~\ref{tab:ablation_mlp_blocks}), 
CKGCNs with linear kernel functions under-perform the variants
with non-linear kernel functions on both ZINC and PATTERN datasets,
even with more learnable parameters.
This observation matches our hypothesis on the indispensability of the non-linearity in kernel functions.
It also justifies the advantage of CKGConv framework over the previous polynomial GNNs which can only introduce linear kernel functions.

\section{Theory and Proof}
\newcommand{\sd}{\text{d}_{G}^{\text{SPD}}}
\newcommand{\rd}{\text{d}_{G}^{\text{RD}}}
\newcommand{\gd}{\text{d}_{G}}

\subsection{The Expressiveness of CKGConv Is Equivalent to GD-WL}
\label{appx:gd_wl}
\newcommand{\multiset}[1]{\{\!\!\{#1\}\!\!\}}
We use a Weisfeiler-Lehman (WL)-like graph isomorphism framework to analyze theoretical expressiveness.
Specifically, we consider the
\emph{Generalized Distance WL} (GD-WL) test, which is based on updating node colors incorporating graph distances proposed by \citet{zhang2023RethinkingExpressivePower}.

For a graph $\mathcal{G} = (\mathcal{V}, \mathcal{E})$,   
the iterative node color update in GD-WL test is defined as:
\begin{align}
\chi^{\ell}_{\mathcal{G}}(v) = \textit{hash}(\{\!\!\{(d_{\mathcal{G}}(v, u), \chi^{\ell-1}_{\mathcal{G}}(u)):u \in \mathcal{V}\}\!\!\})\,.
\label{eq:GD-WL}
\end{align}
where $d_{\mathcal{G}}(v, u)$ denotes a distance between nodes $v$ and $u$, and $\chi_G^0(v)$ is the initial color of $v$. 
The multiset of final node colors $\multiset{\chi_G^L(v): v \in \V}$ at iteration $L$ is hashed to obtain a graph color. 

Our proof for the expressiveness of CKGConv employs the following lemma provided by~\citet{xu2019HowPowerfulAre}. 

\begin{lemma}{(Lemma 5 of ~\citet{xu2019HowPowerfulAre})}
\label{lemma:hash_agg-sp}
     For any countable set $\cX$,
     there exists a function $f:\cX \to \RR^n$ such that
     $h(\hat{\cX}):=\sum_{x \in \hat{\cX}} f(x)$ is unique for each multiset $\hat{\cX} \in \cX$ of bounded size.
     Moreover, for some function $\phi$,
     any multiset function $g$ can be decomposed as $g(\hat{\cX})=\phi(\sum_{x \in \hat{\cX}} f(x))$.
\end{lemma}

\begin{proof}[Proof of Proposition~\ref{prop:ckgconv_gdwl}]
In this proof, we consider shortest-path distance (SPD) as an example of generalized distance (GD). This is denoted as $\sd$
and is assumed to construct the pseudo-coordinates in CKGConv.
The proof holds with other GDs such as the resistance distance (RD)~\cite{zhang2023RethinkingExpressivePower} and RRWP~\cite{ma2023GraphInductiveBiases}, and the choice of GD determines the practical expressiveness of GD-WL.

We consider all graphs with at most $n$ nodes to distinguish in the isomorphism tests.
The total number of possible values of $\gd$ is finite and depends on $n$ (upper bounded by $n^2$).
We define
\begin{align}
    \mathcal{D}_n = \{\sd(u,v): G=(\cV, \cE), |\cV| \leqslant n, u, v\in \cV\}\,,\label{eq:D_n}
\end{align} 
to denote all possible values of $\sd(u,v)$ for any graphs with at most $n$ nodes. 
We note that since $\mathcal{D}_n$ is a finite set, its elements can be listed as $\mathcal{D}_n=\{
d_{G,1},\cdots,d_{G,|\mathcal{D}_n|}\}$.

Then the GD-WL aggregation at the $\ell$-th iteration in Eq.~\eqref{eq:GD-WL} can be equivalently rewritten as (See Theorem E.3 in~\citet{zhang2023RethinkingExpressivePower}):
\begin{align}
   &\chi^{\ell}_G(v):= \text{hash}\Big(\chi_G^{\ell,1}(v), \chi_G^{\ell,2}(v), \cdots, \chi_G^{\ell, |\cD_n|}(v) \Big) \,,\nonumber \\
   \text{where } &\chi_G^{\ell,k}(v):= 
   \multiset{\chi_G^{\ell-1}(u): u \in \cV, \gd(u,v)=d_{G,k} }\,.\label{eq:gd_wl_hash}
\end{align}

In other words, for each node $v$, we can perform a color update by hashing a tuple of color multisets. We construct the $k$-th  multiset by injectively aggregating
the colors of all nodes $u \in \cV$ at a distance $d_{G,k}$ from node $v$.

Assuming the color of each node $\chi^t_G(v)$ is represented as a vector $\bx_v^{(l)} \in \RR^C$, and setting the bias $\mathbf{b}$ to $\mathbf{0}$ for simplicity, the $l$-th CK-GConv layer with a global support (as shown in Eq.~\eqref{eq:ck-gconv-dw}) can be written as
\begin{equation}
    \hat{\mathbf{x}}^{(l)}_v :=  \frac{1}{|\cV|}\sum_{u \in \cV} (\mathbf{W} \mathbf{x}^{(l)}_u) \odot \boldsymbol{\psi}\big(\gd(u,v)\big) \,.   
    \label{eq:ck-gconv-dw-proof}
\end{equation}
where $\boldsymbol{\psi}: \RR \to \RR^C$ and $\mathbf{W} \in \RR^{C \times C}$ is the learnable weight.
Then, we will show that with certain choices of the kernel function, 
a CKGCN is as powerful as GD-WL.

First, we define the kernel function $\boldsymbol{\psi}$ as a composition of $H$ sub-kernel functions $ 
\{\boldsymbol{\psi}^h:\RR \to \RR^F\}_{h=1,\ldots, H}$ such that $\boldsymbol{\psi}(d) = [\boldsymbol{\psi}^1(d) \| \dots \| \boldsymbol{\psi}^H(d)] \in \RR^C, \forall d \in \cD_n$, where $[\cdot \| \cdot]$ denotes the concatenation of vectors and $C=H\cdot F$.

Then Eq.~\eqref{eq:ck-gconv-dw-proof} can be written as
\begin{align}
    \hat{\mathbf{x}}^{(l), h}_v :=&\frac{1}{|\cV|}\sum_{u \in \cV} (\mathbf{W}^h \mathbf{x}^{(l)}_u) \odot \boldsymbol{\psi}^h\big(\gd(u,v)\big) \,,  
    \label{eq:ck-gconv-dw-subkernel} \\
    \hat{\mathbf{x}}^{(l)}_v =&[\hat{\mathbf{x}}^{(l), 1}_v \| \cdots \| \hat{\mathbf{x}}^{(l), H}_v ]\,,
    \label{eq:ck-gconv-dw-concat} 
\end{align}
where $\bW \in \RR^{C \times C}$ is partitioned as $[{\bW^1}^\intercal, \cdots, {\bW^H}^\intercal]^\intercal$ so that each $\bW^h \in \RR^{F \times C}$.

We construct $\boldsymbol{\psi}^h(d):=\mathbb{I}(d=d_{G,h}) \cdot \mathbf{1}$, where $\mathbb{I}:\RR \to \RR$ is the indicator function, $d_{G,h} \in \cD_n$ is a pre-determined condition, and $\mathbf{1} \in \RR^F$.
Then, the convolution by each sub-kernel (Eq.~\eqref{eq:ck-gconv-dw-subkernel} can be written as
\begin{equation}
    \begin{aligned}
        \hat{\mathbf{x}}^{(l), h}_v :=&\frac{1}{|\cV|}\sum_{u \in \cV} (\mathbf{W}^h \mathbf{x}^{(l)}_u) \odot \boldsymbol{\psi}^h\big(\gd(u,v)\big) \,,  \\
        =&\frac{1}{|\cV|}\sum_{u \in \cV} (\mathbf{W}^h \mathbf{x}^{(l)}_u) \odot (\mathbb{I}(\gd(u,v)=d_{G,h})\cdot \mathbf{1} \big) \,, \\
        =&\frac{1}{|\cV|}\sum_{u \in \cV} (\mathbf{W}^h \mathbf{x}^{(l)}_u) \cdot \mathbb{I}(\gd(u,v)=d_{G,h})  \,,\\
        =&\frac{1}{|\cV|} \sum_{\gd(u,v)=d_{G,h}} {\bW}^h \bx^{(l)}_u\,.
    \end{aligned}
    \label{eq:ckgconv_grouped}
\end{equation}
Note that ${\bW}$ can be absorbed as the last layer of the feed-forward network (FFN) in the previous layer.
Because $\bx^{(l)}_u$ is processed by the FFN in the previous layer,
we can invoke Lemma~\ref{lemma:hash_agg-sp} to establish that
each sub-kernel $\boldsymbol{\psi}^h$ (as in Eq.~\eqref{eq:ckgconv_grouped}) can implement an injective aggregating function for 
$\multiset{\chi_G^{t-1}(u): u \in \mathcal{V}, d_G(u,v)=d_{G,h}}$.
The concatenation in Eq.~\eqref{eq:ck-gconv-dw-concat} is an injective mapping of the tuple of multisets $\left(\chi_G^{t,1}, \cdots, \chi_G^{t, |\cD_n|} \right)$.
When any of the linear mappings has irrational weights, 
the projection will also be injective.
Therefore, one CKGConv followed by the FFN can implement the aggregation formula (Eq.~\eqref{eq:gd_wl_hash}), with a sufficiently large number of different $\boldsymbol{\psi}^h$.
Thus, the CKGCN can perform the aggregation of GD-WL.
Therefore, with a sufficiently large number of layers, CKGCN is as powerful as GD-WL in distinguishing non-isomorphic graphs, which concludes the proof.
\end{proof}

\subsection{CKGConv and Equivariant Set Neural Networks}
\label{appx:set_nn}

\begin{proof}[Proof of Proposition~\ref{prop:ckgconv_deepset}]

We prove the proposition for scalar-valued signals, which can be directly generalized to vector-valued signals.

For a globally supported CKGConv, 
given the 1-RRWP after the re-scaling (Appendix~\ref{appx:re-normalization}) 
$P_{i,i}= |\cV| \text{ and } P_{i,j}=0, \forall i,j \in \cV, i\neq j$,
denoted as $P_0$ and $P_1$ for simplicity.
Considering $\psi: \RR \to \RR$ that $\psi(x)=  \gamma \cdot x + \beta$ $\gamma, \beta \in \RR$,
Eq.~\eqref{eq:ck-gconv} can be written as
\begin{equation}
    \begin{aligned}
    (\chi \star \psi)(i) 
    &=\frac{1}{|\cV|} \left( \chi(i) (|\cV|\cdot \gamma + \beta)  + \sum_{j \in \cV; j \neq i} \chi(j) \beta \right) + b\,, \\
    &=\frac{1}{|\cV|} \chi(i) (|\cV| \cdot \gamma + \beta - \beta) + \frac{1}{|\cV|} \sum_{j \in \cV} \chi(j)\beta  +b\,, \\
    &=\gamma \cdot \chi(i) + \beta \cdot \big(\frac{1}{|\cV|} \sum_{j \in \cV} \chi(j)\big) + b\,.
    \label{eq:ck-gconv-app}
    \end{aligned}
\end{equation}
This is the general form of a layer in an \textit{Equivariant Set Network} (Eq. 8 in \citet{segol2020UniversalEquivariantSet}).
This general form can cover a wide range of set neural networks~\cite{zaheer2017DeepSets, qi2017PointNetDeepLearning}.
\end{proof}

\subsection{CKGConv, Polynomial Spectral GNNs and Diffusion Enhaned GNNs}
\label{appx:poly_spectral}

\begin{lemma}\label{lemma:A^k}
   Let $\bA \in \RR^{n \times n}$ denotes the adjacency matrix of an undirected graph $G$ and the diagonal matrix $\bD \in \RR^{n \times n}, [\bD]_{i,i}= \sum_{j \in \cV} [\bA]_{i,j}$ is the degree matrix,  
   the $k$-power of symmetric normalized adjacency matrix $\sA:= \bD^{-1/2} \bA \bD^{-1/2}$ and random walk matrix $\rw:=\bD^{-1} \bA$, 
   satisfy that
   \begin{equation}
       \sA^k  = \bD^{1/2} \rw \bD^{-1/2}, \forall k=1,2, \cdots
\end{equation}
\end{lemma}

\begin{proof}[Proof of Lemma~\ref{lemma:A^k}]

For arbitrary $k \geq 1$, we have
\begin{equation}
    \begin{aligned}
        \sA^k &= (\bD^{1/2}\rw \bD^{-1/2})^k \\
        &=  (\bD^{1/2}\rw \bD^{-1/2}) (\bD^{1/2}\rw \bD^{-1/2})\cdots (\bD^{1/2}\rw \bD^{-1/2}) \text{ 
    ,$k$ times} \\
        &= \bD^{1/2} \rw \cdots \rw \bD^{-1/2} \\
        &= \bD^{1/2} \rw^k  \bD^{-1/2} \, . 
    \end{aligned}
    \end{equation}
\end{proof}

\begin{proof}[Proof of Proposition~\ref{prop:ckgconv_spectral}] 
\newcommand{\x}{\mathbf{x}}
\renewcommand{\y}{\mathbf{y}}

Irrespective of the specific polynomial parameterization that is employed, any $K{-}1$ order \textit{Polynomial Spectral Graph Neural Network} can be defined in a general form with $\L =\mathbf{I} - \D^{-1/2}\A\D^{-1/2} \in \RR^{n \times n}$, 
parameterized by a learnable vector $\mathbf{\theta} = [\theta_0, \cdots, \theta_{K-1}]^\intercal \in \mathbb{R}^{K \times 1}$ for the
filtering of an input graph signal $\x \in \mathbb{R}^{n \times 1}$
to obtain an output graph signal $\y \in \mathbb{R}^{n \times 1}$, as follows:
\begin{align}
\y &= g_{\mathbf{\theta}}(\L)\x \,,  \nonumber\\
&= \sum_{k=0}^{K-1}\theta_k \L^k \x \,, \nonumber\\
&= \sum_{k=0}^{K-1} \theta_k \sum_{r=0}^k  
\binom{k}{r}(-1)^r 
\sA^r \x\,, \nonumber\\
&= \sum_{k=0}^{K-1}\theta'_k \sA^k \x\,. 
\label{eq:poly_spectral_gnn}
\end{align}

Here $\sA=\D^{-1/2}\A\D^{-1/2}$ and  $\theta'_k = \sum_{r=k}^{K-1} \binom{r}{k} (-1)^k\theta_r$. 
Therefore, the spectral filter  $g_{\mathbf{\theta}}(\L)$ can be represented by a linear combination of a collection of polynomial bases $\{\mathbf{I}, \sA^1, \sA^2, \cdots, \sA^{K-1}\}$.
In other words,
\begin{equation}
\begin{aligned}
   [g_{\mathbf{\theta}}(\L)]_{i,j} &= \psi([\mathbf{I}, \sA^1, \sA^2, \cdots, \sA^{K-1}]_{i,j})\,, \\
  &=  
 \psi(d_i^{1/2} [\mathbf{I}, \rw^1, \rw^2, \cdots, \rw^{K-1}]_{i,j} d_j^{-1/2})\,, \text{ using Lemma~\ref{lemma:A^k}} \\
  &=  
   d_i^{1/2} \psi([\mathbf{I}, \rw^1, \rw^2, \cdots, \rw^{K-1}]_{i,j}) d_j^{-1/2} \,, \text{ as $\psi$ is a linear projection}\\
   &=d_i^{1/2} \psi(\P_{i,j}) d_j^{-1/2}\,,\\
   &=\frac{1}{S} \cdot d_i^{1/2} \psi(S\cdot \P_{i,j}) d_j^{-1/2}\,.\\
\end{aligned}
\end{equation}
Here $\psi: \RR^K \to \RR$ is a linear projection; $d_i = \D_{i,i} \in \RR$ is the degree of node $i$; and $S \in \RR$ is the scaling term in the scaled-convolution design and the RRWP rescaling.

In other words, with the $K$-RRWP as pseudo-coordinates, 
CKGConv with a linear kernel $\psi$ can recover most \emph{polynomial spectral GNNs} in the form of Eq.~\eqref{eq:poly_spectral_gnn} irrespective of the specific polynomial parameterization that is used, if $d_i^{1/2}$ and $d_j^{-1/2}$ are injected properly, for all $i,j \in \cV$.
The result trivially holds for other Laplacian normalizations (e.g., row-normalized, max-eigenvalue normalized), where different constant multipliers are injected via adaptive degree scalers to $\P_{i,j}$.

Similarly, \textit{Polynomial Diffusion Enhanced Graph Neural Networks} employing polynomials of $\sA$ or its variants also can be represented by Eq.~\eqref{eq:poly_spectral_gnn}. Hence, the result follows.
\end{proof}

\subsection{Degree Information and Normalization Layers}
\label{appx:deg_and_LN}

Normalization layers are essential for deep neural networks.
\citet{ma2023GraphInductiveBiases} provide a thorough discussion on the impact of normalization layers on the explicit injected degree information via sum-aggregation or degree scalers,
which  
motivates our choice of BatchNorm~\cite{ioffe2015BatchNormalizationAccelerating} over LayerNorm~\cite{ba2016LayerNormalization}.

\begin{proposition}{\cite{ma2023GraphInductiveBiases}}\label{prop:ln_degree}
    Sum-aggregated node representations, degree-scaled node representations, and mean-aggregated node representations all have the same value after the application of a LayerNorm on node representations.
\end{proposition}

\begin{proof}[Proof of Proposition~\ref{prop:ln_degree}]
Regardless the linear transformation in MPNN shared by nodes,
we can write the output representation for a node $i$ from a sum-aggregator as 
$\mathbf{x}^\text{sum}_i =  d_i \cdot \mathbf{x}^\text{mean}_i$, where $d_i \in \mathbb{R}$ is the degree of node $i$ and $\mathbf{x}^\text{mean}_i=[x_{i1}, \dots x_{iF}]^\top \in \mathbb{R}^F$ is the node representation from a mean-aggregator.
The layer normalization statistics for a node $i$ over all hidden units are computed as follows:
\begin{equation}
    \begin{aligned}
    &\mu^\text{sum}_i = \frac{1}{F} \sum_{j=1}^F x^\text{sum}_{ij} = 
    \frac{1}{F} \sum_{j=1}^F d_i \cdot x^\text{mean}_{ij} 
    = \frac{d_i}{F} \sum_{j=1}^F x^\text{mean}_{ij} 
    = d_i \cdot \mu^\text{mean}_i \\
    &\sigma^\text{sum}_i =
    \sqrt{\frac{1}{F}\sum_{j=1}^F (x^\text{sum}_{ij}-\mu^\text{sum})^2}
    = 
    \sqrt{\frac{d_i^2}{F}\sum_{j=1}^F (x^\text{mean}_{ij}-\mu^\text{mean})^2} = d_i \cdot \sigma_i^\text{mean}
    \end{aligned}
\end{equation}

Therefore, regardless of the elementwise affine transforms shared by all nodes, each element of 
the normalized representation 
\begin{equation}
    \begin{aligned}
        \tilde{x}_{ij}^\text{sum} = 
        \frac{({x}_{ij}^\text{sum} - \mu_i^\text{sum})}{\sigma^\text{sum}_i} 
        = 
        \frac{(d_i \cdot {x}_{ij}^\text{mean} - d_i\cdot \mu_i^\text{mean})}{d_i \cdot \sigma^\text{mean}_i}
        = 
        \frac{( {x}_{ij}^\text{mean} - \mu_i^\text{mean})}{\sigma^\text{mean}_i} = \tilde{x}_{ij}^\text{mean},
        \quad \forall i \in \mathcal{V}, \forall j=1,\dots, F, 
    \end{aligned}
\end{equation}
is the same for both sum-aggregation and mean-aggregation. 

The same conclusion can be seen for degree scalers, by simply changing $d_i$ to $f(d_i)$ in the proof, where $f: \mathbb{R} \to \mathbb{R}_{> 0}$.
\end{proof}

Note that, BatchNorm does not have such an impact on degree information,
since the normalization statistics are computed across all nodes (with different degrees) 
in each mini-batch per channel.
